\documentclass[ba]{imsart}

\pubyear{2023}
\volume{05}
\issue{05}

\usepackage{amsthm}
\usepackage{amsmath}
\usepackage{natbib}
\usepackage[colorlinks,citecolor=blue,urlcolor=blue,filecolor=blue,backref=page]{hyperref}
\usepackage{graphicx}

\usepackage{microtype}
\usepackage{graphicx}
\usepackage{subfigure}
\usepackage{booktabs} 

\usepackage[utf8]{inputenc} 
\usepackage[T1]{fontenc}    
\usepackage{hyperref}       
\usepackage{booktabs}       
\usepackage{amsfonts}       
\usepackage{nicefrac}       
\usepackage{microtype}      
\usepackage{amsthm}	
\usepackage{amsbsy} 
\usepackage{amsmath}
\usepackage{graphicx}
\usepackage{adjustbox}
\usepackage{bm}
\usepackage{wrapfig}
\usepackage{natbib}
\usepackage[textsize=tiny,textwidth=1.0cm]{todonotes}
\RequirePackage{algorithm}
\RequirePackage{algorithmic}
\usepackage{soul}

\usepackage{hyperref}

\makeatletter
\newcommand*{\addFileDependency}[1]{
  \typeout{(#1)}
  \@addtofilelist{#1}
  \IfFileExists{#1}{}{\typeout{No file #1.}}
}
\makeatother

\usepackage{cleveref}

\newtheorem{lemma}{Proposition}

\newcommand{\D}{\mathcal{D}}

\startlocaldefs
\numberwithin{equation}{section}
\theoremstyle{plain}

\endlocaldefs
\begin{document}

\begin{frontmatter}
\title{Variational Inference for Bayesian Neural Networks under Model and Parameter Uncertainty}
\runtitle{VI for BNN under Model and Parameter Uncertainty}

\begin{aug}
\author{\fnms{Aliaksandr} \snm{Hubin}\thanksref{addr1}
\ead[label=e1]{aliaksah@math.uio.no}},
\author{\fnms{Geir} \snm{Storvik}\thanksref{addr2}
\ead[label=e2]{geirs@math.uio.no}}

\runauthor{Hubin and Storvik}

\address[addr1]{Department of Mathematics, University of Oslo, NMBU, NR, HiOF.
    \printead{e1} 
}

\address[addr2]{Department of Mathematics, University of Oslo, NR. 
    \printead{e2}
}

\end{aug}
\begin{abstract}
Bayesian neural networks (BNNs) have recently regained a significant amount of attention in the deep learning community due to the development of scalable approximate Bayesian inference techniques. There are several advantages of using a Bayesian approach: Parameter and prediction uncertainties become easily available, facilitating rigorous statistical analysis. Furthermore, prior knowledge can be incorporated. However, so far, there have been no scalable techniques capable of combining both structural and parameter uncertainty. In this paper, we apply the concept of model uncertainty as a framework for structural learning in BNNs and hence make inference in the joint space of structures/models and parameters. 
Moreover, we suggest an adaptation of a scalable variational inference approach with reparametrization of marginal inclusion probabilities to incorporate the model space constraints. Experimental results on a range of benchmark datasets show that we obtain comparable accuracy results with the competing models, but based on methods that are much more sparse than ordinary BNNs. 
\end{abstract}


\end{frontmatter}


\section{Introduction}

In recent years, frequentist deep learning procedures have become extremely popular and highly successful in a wide variety of real-world applications ranging from natural language to image analyses \citep{Goodfellow-et-al-2016}. These algorithms iteratively apply some nonlinear transformations aiming at optimal prediction of response variables from the outer layer features. This yields high flexibility in modelling complex conditional distributions of the responses. Each transformation yields another hidden layer of features which are also called neurons. The architecture/structure of a deep neural network includes the specification of the nonlinear intra-layer transformations (\emph{activation functions}), the number of layers (\emph{depth}), the number of features at each  layer (\emph{width}) and the connections between the neurons (\emph{weights}). In the standard (frequentist) settings, the resulting model is trained using some optimization procedure (e.g. stochastic gradient descent) with respect to its parameters in order to fit a particular objective (like minimization of the root mean squared error or negative log-likelihood).
Very often deep learning procedures outperform traditional statistical models, even when the latter are carefully designed and reflect expert knowledge~\citep{refenes1994stock, razi2005comparative, adya1998ective, sargent2001comparison, kanter2015deep}. However, typically one has to use huge datasets to be able to produce generalizable neural networks and avoid overfitting issues. Even though several regularization techniques ($L_1$ and $L_2$ penalties on the weights, dropout, batch normalization, etc.) have been developed for deep learning procedures to avoid overfitting to training datasets, the success of such approaches is not obvious.  Unstuctured pruning the network, either by putting some weights to zero or by removing some nodes has been shown to be possible.

As an alternative to frequentist deep learning approaches, Bayesian neural networks represent a very flexible class of models, which are quite robust to overfitting \citep{neklyudov2018variance}. However, they often remain heavily over-parametrized. There are several implicit approaches for the sparsification of BNNs by shrinkage of weights through priors~\citep{JMLR:v15:jylanki14a, blundell2015weight, molchanov2017variational,ghosh2018structured, neklyudov2017structured}. For example, \citet{blundell2015weight} suggest a mixture of two Gaussian densities and then perform a fully factorizable mean-field variational approximation. \citet{ghosh2019model, louizos2017bayesian} independently generalize this approach by means of suggesting Horseshoe priors \citep{carvalho2009handling} for the weights, providing even stronger shrinkage and automatic specification of the mixture component variances required in~\citet{blundell2015weight}. 
Some algorithmic procedures can also be seen to correspond to specific Bayesian priors, e.g. \citet{molchanov2017variational} 
show that Gaussian dropout corresponds to BNNs with log uniform priors on the weight parameters.

In this paper, we consider a formal Bayesian approach for jointly taking into account \textit{structural}  \textit{uncertainty} and \textit{parameter uncertainty} in BNNs as a generalization of the methods developed for Bayesian model selection within linear regression models. The approach is based on introducing latent binary variables corresponding to the inclusion-exclusion of particular weights within a given architecture. This is done by means of introducing spike-and-slab priors. 
Such priors for the BNN setting were suggested in~\citet{polson2018posterior} and~\citet{hubin2018thesis}. A computational procedure for inference in such settings was proposed in an early version of this paper~\citep{hubin2019combining} and was further used without any changes in~\citet{bai2020efficient}. An asymptotic theoretical result for the choice of prior inclusion probabilities is additionally presented in~\citet{bai2020efficient}, however, it only is valid when the number of parameters goes to infinity and when the prior variance of the slab components is fixed. Here, we go further, in that we introduce hyperpriors on the prior inclusion probabilities and variance of the slab components making them stochastic, we also allow for more flexible variational approximations based on the multivariate Gaussian structures for inclusion indicators. Additionally, we consider several alternative prediction procedures for fully Bayesian model averaging, including posterior mean-based models, and the median probability model, and perform a comprehensive experimental study comparing the suggested approach with several competing algorithms and several data sets.

Using a Bayesian formalization in the space of models allows adapting the whole machinery of Bayesian inference in the joint model-parameter settings, including \textit{Bayesian model averaging} (BMA) (across all models) or \textit{Bayesian model selection} (BMS) of one \textit{“best”} model with respect to some model selection criterion \citep{claeskens2008model}. In this paper, we study BMA as well as the \textit{median probability} model \citep{barbieri2004optimal, barbieri2018median} and \textit{posterior mean} model-based inference for BNNs. Sparsifying properties of BMS (in particular the median probability model) are also addressed within the experiments. Finally, following \citet{hubin2018thesis} we will link the obtained \textit{marginal inclusion probabilities} to \textit{binary dropout} rates, which gives proper probabilistic reasoning for the latter. The inference algorithm is based on scalable stochastic variational inference. 

The suggested approach has similarities to binary dropout that has become very popular~\citep{srivastava2014dropout}. However, while standard binary dropout can be seen as a Bayesian approximation to a Gaussian process model where only parameter estimation is taken into account~\citep{Gal2016Uncertainty}, our approach explicitly \emph{models} structural uncertainty. In this sense, it is closely related to Concrete dropout \citep{gal2017concrete}. However, the model proposed by \citet{gal2017concrete} does not allow for BMS: The median probability model will either select all weights or nothing due to a strong assumption of having the same dropout probabilities for the whole layer. Furthermore, the variational approximation procedure applied in \citet{gal2017concrete} has not been studied in the model uncertainty context. 

At the same time,  it is important to state explicitly that our approach does not currently aim at interpretable inference on Bayesian neural networks in most of the cases with an exception of having a special case of model selection in GLM models. Also, interpretable models could in principle be feasible if direct connections from all the layers to the responses are allowed. However, such cases are not addressed in this paper.

The rest of the paper is organized as follows: Section~\ref{sec:back} gives some background and discuss related work. The class of BNNs and the corresponding model space are mathematically defined in Section~\ref{section2}. In Section~\ref{section3}, we describe the algorithm for training the suggested class of models using the reparametrization of marginal inclusion probabilities. Section~\ref{OtherInf} discusses several predictive inference possibilities.  In Section~\ref{section4}, the suggested approach is applied to the two classical benchmark datasets MNIST, FMNIST (for image classifications) as well as  PHONEME (for sound classification). We also compare the results with some of the existing approaches for inference on BNNs. Finally, in Section~\ref{section5} some conclusions and suggestions for further research are given. Additional results are provided in the supplementary materials to the paper.


\section{Background and related work}\label{sec:back}

Bayesian neural networks (BNNs) were already introduced a few decades ago by~\citet{neal1992bayesian,mackay1995bayesian,bishop1997bayesian}. 
BNNs take advantage of the rigorous Bayesian approach and are able to properly handle parameter and prediction uncertainty and can in principle also incorporate prior knowledge. In many cases, this leads to more robust solutions with less overfitting. However, this comes at a price of extremely high computational costs. Until recently, inference on BNNs could not scale to large and high-dimensional data due to the limitations of standard MCMC approaches, the main numerical procedure in use. Several attempts based on subsampling techniques for MCMC, which are either approximate \citep{bardenet2014towards, bardenet2017markov, korattikara2014austerity, quiroz2014speeding, welling2001belief} or exact \citep {quiroz2016exact, maclaurin2014firefly, liu2015exact,welling2001belief} have been proposed, but none of them is able to explore the parameter spaces efficiently in ultrahigh-dimensional settings. 

An alternative to the MCMC technique is to perform approximate Bayesian inference through variational Bayes, also known as variational inference \citep{jordan1999introduction}. Due to the fast convergence properties of the variational methods, variational inference algorithms are typically orders of magnitude faster than MCMC algorithms  in high-dimensional problems~\citep{ahmed2012scalable}. The variational inference has various applications in latent variable models, such as mixture models \citep{humphreys2000approximate}, hidden Markov models \citep{mackay1997ensemble} and graphical models \citep{attias2000variational} in general. \citet{graves2011practical} suggested the methodology for scalable variational inference to Bayesian neural networks. This methodology was further improved by incorporating various variance reduction techniques, which are discussed in \citet{Gal2016Uncertainty}.


As mentioned in the introduction, it has been shown that the majority of the weight parameters in neural netwoeks can be pruned out from the model without a significant loss of predictive accuracy. However, pruning is typically done implicitly by deleting the weights via ad-hoc thresholding. 
Yet, learning which parameters to include in a model, can also be framed as a structure learning or a model selection problem.

As discussed in \citet{claeskens_hjort_2008}, \citet{steel2020model} or \citet{hansen2001model} (among other venues),  model selection and model averaging in statistics generally assumes a discrete and countable (finite or infinite) set of models living on a corresponding model space. Models within a model space can differ in terms of the likelihood used, the link functions addressed, or which parameters are included in a linear or non-linear predictor. The purpose of model selection is to choose a single (best in some sense) model from a model space. This choice can lead to more interpretable models in some use cases  \citep[like variable selection in linear regression,][]{kuo1998variable} or simply best models for some purpose (like prediction) in others \citep{geisser1979predictive}, sometimes both coincide \citep{hubin2022flexible}, but that is not always the case \citep{breiman2001statistical}. Sparsity in model selection may or may not be of interest dependent on the context, although parsimony in some sense is typically desired. At the same time, the model selection often leads to problems with uncertainty handling resulting in too narrow confidence/credible intervals of the parameters  \citep{heinze2018variable} and thus often may result in similar problems for predictions, i.e. overfitting. Model averaging (if model uncertainty is properly addressed) can resolve these issues \citep{bornkamp2017model} and has other advantages \citep{steel2020model}. Also, model uncertainty aware model selection, e.g. using the median probability model is more robust to overfitting \citep{ghosh2015bayesian}.  

There have been numerous works showing the efficiency and accuracy of model selection/averaging related to parameter selection through introducing latent variables corresponding to different discrete model configurations. In the Bayesian context, the posterior distribution can then be used to both select the best sparse configuration and address the joint model-and-parameters-uncertainty explicitly \citep{george1993variable,Clyde:Ghosh:Littman:2010,Frommlet2012,hubin2018mode,hubin2020novel,hubin2022flexible}. Spike-and slab priors~\citep{mitchell1988bayesian} are typically used in this setting.  All of these approaches have demonstrated both good predictive performance of the obtained sparse models and the ability to recover meaningful complex nonlinearities. They are however based on adaptations of Markov chain Monte Carlo (MCMC) and do not scale well to large high-dimensional data samples. 
\citet{louizos2017bayesian} also warn about the complexity of explicit discretization of model configuration within BNNs, as it causes an exponential explosion with respect to the total number of parameters, and hence infeasibility of inference for high-dimensional problems. \citet{polson2018posterior} study the use of the spike-and-slab approach in BNNs from a theoretical standpoint.


\citet{logsdon2010variational, carbonetto2012scalable} suggest a fully-factorized variational distribution capable of efficiently and precisely "linearizing" the computational burden of Bayesian model selection in the context of \emph{linear} models with an ultrahigh number of potential covariates, typical for genome-wide association studies (GWAS). In the discussion of his Ph.D. thesis, \citet{hubin2018thesis} proposed combining the approaches of \citet{logsdon2010variational, carbonetto2012scalable} and  \citet{graves2011practical}  for scalable approximate Bayesian inference on the joint space of models and parameters in deep Bayesian regression models. We develop this idea further in this article. 

\section{The model}\label{section2}
A neural network model links (possibly multidimensional) observations $\boldsymbol y_i\in \mathcal{R}^r$ and explanatory variables $\boldsymbol x_i\in \mathcal{R}^p$   via a probabilistic functional mapping with a mean parameter vector $\boldsymbol \mu_i = \boldsymbol \mu_i(\boldsymbol x_i) \in \mathcal{R}^r$: 
\begin{align}
\boldsymbol y_i \sim \mathfrak{f}\left(\boldsymbol \mu_i(\boldsymbol x_i), \phi \right),\quad i \in \{1,...,n\}, \label{eq:ann} 
\end{align}
where $\mathfrak{f}$ is some observation distribution, typically from the exponential family, while $\phi$ is a dispersion parameter. To construct the vector of mean parameters $\boldsymbol \mu_i$, one builds a sequence of building blocks of hidden layers through semi-affine  transformations:
\begin{align}
z^{(l+1)}_{ij} =&g^{(l)}_j\left(\beta^{(l)}_{0j} + \sum_{k=1}^{p^{(l)}}\beta^{(l)}_{kj}z^{(l)}_{ik}\right), l=1,...,L-1, j=1,...,p^{(l+1)},\label{eq:neuron}
\end{align}
with $\mu_{ij}=z_{ij}^{(L)}$. Here, $L$ is the number of layers, $p^{(l)}$ is the number of nodes within the corresponding layer while $g^{(l)}_{j}$ is a univariate function (further referred to as the \textit{activation function}).
Further, $\beta^{(l)}_{kj}\in\mathcal{R}, k > 0$ are the weights (slope coefficients) for the inputs $z^{(l)}_{ik}$ of the $l$-th layer (note that $z^{(1)}_{ik} = x_{ik}$ and $p^{(1)}=p$). For $k=0$, we obtain the intercept/bias terms. Finally, we introduce latent binary indicators  $\gamma^{(l)}_{kj}\in \{0,1\}$  switching the corresponding weights on and off such that $\beta_{kj}^{(l)}=0$ if $\gamma^{(l)}_{kj}=0$.  


In our notation, we explicitly differentiate between discrete structural/model configurations defined by the vectors 
$\boldsymbol\gamma =\{\gamma_{kj}^{(l)},j=1,..,p^{(l+1)},k=0,...,p^{(l)},l=1,...,L-1\}$ 
(further referred to as models) constituting the model space $\Gamma$ and parameters of the models, conditional on these configurations $\boldsymbol \theta|\boldsymbol\gamma = \{\boldsymbol\beta,\phi|\boldsymbol\gamma\}$, where only those $\beta_{kj}^{(l)}$ for which $\gamma_{kj}^{(l)}=1$ are included. 
This approach is (in statistical science literature) a rather standard  way to explicitly specify the model uncertainty in a given class of models and is used in e.g. \citet{Clyde:Ghosh:Littman:2010,Frommlet2012,hubin2022flexible}. 

A Bayesian approach is completed by specification of model priors $p(\boldsymbol\gamma)$ and parameter priors for each model $p(\boldsymbol\beta|\boldsymbol\gamma,\phi)$. If the dispersion parameter is present in the distribution of the outcomes, one also has to define a prior
$p(\phi|\boldsymbol\gamma)$. Many kinds of priors on $p(\boldsymbol\beta|\boldsymbol\gamma,\phi)$ can be considered, including the mixture of Gaussians prior \citep{blundell2015weight}, the Horseshoe prior \citep{ghosh2019model,louizos2017bayesian}, or mixtures of g-priors \citep{li2018mixtures}, which could give further penalties to the weight parameters. 
We first following our early preprint \citep{hubin2019combining} as well as even earlier ideas from \citet{hubin2018thesis,polson2018posterior} and consider the independent Gaussian spike-and-slab weight priors  combined with independent Bernoulli priors for the latent inclusion indicators being equal to 1.  This choice of the priors corresponds to marginal spike-and-slab priors for the weights~\citep{Clyde:Ghosh:Littman:2010}:
\begin{subequations}\label{eq:prior.beta}
\begin{alignat}{4}
p(\beta^{(l)}_{kj}|\sigma_{\beta,l}^2,\gamma^{(l)}_{kj}) =& \gamma^{(l)}_{kj}\mathcal{N}(0,\sigma_{\beta, l}^2) + (1-\gamma^{(l)}_{kj})\delta_0(\beta^{(l)}_{kj}),\\ 
p(\gamma^{(l)}_{kj})=&\text{Bernoulli}(\psi^{(l)}),
\end{alignat}
\end{subequations}
Here, $\delta_0(\cdot)$ is the delta mass or "spike" at zero, $\sigma_{\beta,l}^2$ is the prior variance of $\beta^{(l)}_{kj}$, whilst $\psi^{(l)}\in(0,1)$ is the prior probability for including the weight $\beta^{(l)}_{kj}$ into the model. 

To automatically infer the prior variance and the prior probability for including a weight, we assume for $\sigma_{\beta,l}^2$ a standard  inverse Gamma hyperprior with hyperparameters $a_\beta^{(l)},b_\beta^{(l)}$, and for $\psi^{(l)}$ a $\text{Beta}(a^{(l)}_\psi,b^{(l)}_\psi)$ prior:
\begin{subequations}\label{eq:prior.beta}
\begin{alignat}{4}
p(\beta^{(l)}_{kj}|\sigma_{\beta,l}^2,\gamma^{(l)}_{kj}) =& \gamma^{(l)}_{kj}\mathcal{N}(0,\sigma_{\beta, l}^2) + (1-\gamma^{(l)}_{kj})\delta_0(\beta^{(l)}_{kj}),\\ 
p(\sigma^{2}_{\beta, l}) =& \text{Inv-Gamma}(a^{(l)}_\beta,b^{(l)}_\beta),\\
p(\gamma^{(l)}_{kj})=&\text{Bernoulli}(\psi^{(l)}),\\
p(\psi^{(l)}) =& \text{Beta}(a^{(l)}_\psi,b^{(l)}_\psi).
\end{alignat}
\end{subequations}
Here, $\delta_0(\cdot)$ is the delta mass or "spike" at zero, $\sigma_{\beta,l}^2$ is the prior variance of $\beta^{(l)}_{kj}$ with a standard inverse Gamma hyperprior with hyperparameters $a_\beta^{(l)},b_\beta^{(l)}$, whilst $\psi^{(l)}\in[0,1]$ is the prior probability for including the weight $\beta^{(l)}_{kj}$ into the model. Further, $\psi^{(l)}$ is assumed $\text{Beta}(a^{(l)}_\psi,b^{(l)}_\psi)$ distributed. With the presence of a dispersion parameter, an additional prior is needed, see \citet{dey2000generalized}. We will refer to our model as the Latent Binary Bayesian Neural Network (LBBNN) model. 


\section{Bayesian inference}\label{section3}

The main goal of inference with uncertainty in both models and parameters is to infer the posterior marginal distribution of some parameter of interest $\Delta$ 
(for example the distribution of a new observation $y^*$ conditional on new covariates $\boldsymbol{x}^*$) based on data $\D$:
\begin{align}\label{eq:pred.delta}
p(\Delta|\D) = \sum_{\boldsymbol\gamma \in \Gamma}\int_{\boldsymbol\theta\in\Theta_\gamma}p(\Delta|\boldsymbol\theta,\boldsymbol\gamma,\D)p(\boldsymbol\theta,\boldsymbol\gamma|\D)\text{d}\boldsymbol\theta,
\end{align}
where $\Theta_\gamma$ is the parameter space defined through $\bm\gamma$.
Standard procedures for dealing with complex posteriors is to apply Monte Carlo methods which involve simulations from $p(\boldsymbol\theta,\boldsymbol\gamma|\D)$. For the  model defined by~\eqref{eq:ann}-\eqref{eq:prior.beta} with many hidden variables, such simulations become problematic.

The idea behind variational inference~\citep{graves2011practical,blei2017variational}  is to apply the approximation
\begin{align}\label{eq:pred.delta2}
\tilde p(\Delta|\D) = \sum_{\boldsymbol\gamma \in \Gamma}\int_{\boldsymbol\theta\in\Theta_\gamma}p(\Delta|\boldsymbol\theta,\boldsymbol\gamma,\D)q_{\bm\eta}(\boldsymbol\theta,\boldsymbol\gamma)\text{d}\boldsymbol\theta.
\end{align}
for some suitable (parametric) distribution
$q_{{\boldsymbol\eta}}(\boldsymbol\theta,\boldsymbol\gamma)$ which, with appropriate choices of the parameters $\boldsymbol\eta$, approximates the posterior well and is \emph{simple} to sample from.
The specification of $\bm\eta$ is typically obtained through the minimization of
the Kullback-Leibler divergence from the variational family distribution to the posterior distribution:
\begin{align}
\text{KL}(q_{\boldsymbol\eta}(\boldsymbol\theta,\boldsymbol\gamma)&||p(\boldsymbol\theta,\boldsymbol\gamma|\D))
 = \sum_{\boldsymbol\gamma \in \Gamma}\int_{\Theta_\gamma}q_{\boldsymbol\eta}(\boldsymbol\theta,\boldsymbol\gamma)\log \tfrac{q_{\boldsymbol\eta}(\boldsymbol\theta,\boldsymbol\gamma)}{p(\boldsymbol\theta,\boldsymbol\gamma|\D)}\text{d}\boldsymbol\theta,\label{eq:KL}
\end{align}
 with respect to the variational parameters $\boldsymbol\eta$. Compared to standard variational inference approaches, the setting is extended to include the discrete model identifiers $\bm\gamma$. For an optimal choice $\hat{\bm\eta}$ of $\bm\eta$, inference on $\Delta$ is performed through Monte Carlo estimation of~\eqref{eq:pred.delta2} inserting 
 $\hat{\bm\eta}$ for $\bm\eta$. The main challenge then becomes choosing a suitable variational family  and a computational procedure for minimizing~\eqref{eq:KL}. Note that although this minimization is still a computational challenge, it will typically be much easier than directly obtaining samples from the true posterior. The final Monte Carlo estimation will be simple, provided the variational distribution  $q_{\bm\eta}(\bm\theta,\bm\gamma)$ is selected such that it is simple to sample from.

As in standard settings of variational inference, minimization of the divergence~\eqref{eq:KL} is equivalent to  maximization of the evidence lower bound (ELBO)
\begin{align}
\mathcal{L}_{VI}(\boldsymbol\eta)=& \sum_{\boldsymbol\gamma \in \Gamma}\int_{\Theta_\gamma}q_{\boldsymbol\eta}(\boldsymbol\theta,\boldsymbol\gamma)\log p(\D|\boldsymbol\theta,\boldsymbol\gamma) \text{d}\boldsymbol\theta-\text{KL}\left(q_{\boldsymbol\eta}(\boldsymbol\theta,\boldsymbol\gamma)||p(\boldsymbol\theta,\boldsymbol\gamma)\right) \label{ELBO}
\intertext{through the equality}
\mathcal{L}_{VI}(\boldsymbol\eta)=&p(\D) -\text{KL}\left(q_{\boldsymbol\eta}(\boldsymbol\theta,\boldsymbol\gamma)||p(\boldsymbol\theta,\boldsymbol\gamma|\D)\right),\nonumber
\end{align}
which also shows that $\mathcal{L}_{VI}(\boldsymbol\eta)$ is a lower bound of the marginal likelihood $p(\D)$.

\subsection{Variational distributions}

We will consider a variational family previously proposed for linear regression~\citep{logsdon2010variational, carbonetto2012scalable} which we extend to the LBBNN  setting. 
Assume
\begin{align}
q_{{\boldsymbol\eta}}(\boldsymbol\theta,\boldsymbol\gamma) = q_{\boldsymbol\eta_0}(\phi) \prod_{l=1}^{L-1}\prod_{j=1}^{p^{(l+1)}}\prod_{k=0}^{p^{(l)}}q_{\kappa_{kj},\tau_{kj}}(\beta^{(l)}_{kj}|\gamma^{(l)}_{kj})q_{{\alpha^{(l)}_{kj}}}(\gamma^{(l)}_{kj}),\label{eq:varfactms}
\end{align}
where $q_{\boldsymbol\eta_0}(\phi)$ is some appropriate distribution for the dispersion parameter, 
\begin{align}
q_{\kappa^{(l)}_{kj},\tau^{(l)}_{kj}}\left(\beta^{(l)}_{kj}|\gamma^{(l)}_{kj}\right) =&\gamma^{(l)}_{kj}\mathcal{N}(\kappa^{(l)}_{kj},{\tau^2}^{(l)}_{kj}) + (1- \gamma^{(l)}_{kj})\delta_0(\beta^{(l)}_{kj}),\label{par_approx}
\intertext{and}
q_{{\alpha^{(l)}_{kj}}}(\gamma^{(l)}_{kj}) =&
\text{Bernoulli}(\alpha^{(l)}_{kj}).\label{stuct_approx}
\end{align}
With probability, $\alpha^{(l)}_{kj}\in [0,1]$, the posterior of parameters of weight $\beta^{(l)}_{kj}$ will be approximated by a normal distribution with some mean and variance ("slab"), and otherwise, the weight is put to zero.
Thus, $\alpha^{(l)}_{kj}$ will approximate the marginal posterior inclusion probability of the weight $\beta^{(l)}_{kj}$.
Here, $\boldsymbol\eta= \{{\boldsymbol\eta_0,(\kappa^{(l)}_{kj},\tau^2}^{(l)}_{kj},\alpha^{(l)}_{kj}),l=1,...,L-1,k=1,...,p^{(l+1)},j=1,...,p^{(l)}\}$. 

A similar variational distribution has also been considered within BNN through the dropout approach~\citep{srivastava2014dropout}. For dropout, however, the final network is dense but trained through a  Monte Carlo average of sparse networks.
In our approach, the target distribution is different in the sense of including the binary variables $\{\gamma^{(l)}_{kj}\}$ as part of the \emph{model}. Hence, our marginal inclusion probabilities can serve as a particular case of {dropout} rates with a \textit{proper} probabilistic interpretation in terms of structural model uncertainty.

The variational distribution~\eqref{eq:varfactms}-\eqref{stuct_approx}, corresponding to the commonly applied mean-field approximation, can be seen as a rather crude approximation, which completely ignores all posterior dependence between the model structures or parameters. Consequently, the resulting conclusions can be misleading or inaccurate as the posterior probability of one weight might be highly affected by the inclusion of others. Such a dependence structure can be built into the variational approximation either through the $\gamma$'s or through the $\beta$'s (or both). Here, we only consider dependence structures in the inclusion variables. We still assume independence between layers, but within layers, we introduce a dependence structure by defining
$\bm\alpha^l=\{\alpha_{kj}^{(l)}\}$ now to be a stochastic vector, which on logit-scale follows a multivariate normal distribution:
\begin{align}
\text{logit}(\boldsymbol\alpha^{(l)}) \sim \text{MVN}(\boldsymbol\xi^{(l)},\boldsymbol\Sigma^{(l)}).\label{stuct_depend}
\end{align}
 Here, either a full covariance matrix $\boldsymbol\Sigma^{(l)}$ or a low-rank parametrization for the covariance is possible. For the latter, $\boldsymbol\Sigma^{(l)}  = \boldsymbol F^{(l)}{\boldsymbol F^{(l)}}^T + \boldsymbol D^{(l)}$ with $\boldsymbol F^{(l)}$ being the factor part of low-rank form of covariance matrix and $D^{(l)}$ is the diagonal part of low-rank form of covariance matrix. This drastically reduces the number of parameters  and allows for efficient computations of the determinant and inverse matrix. 
A particularly interesting case is when $\boldsymbol F^{(l)}$ has rank zero in which case we retain independence between the components but some penalization in the variability of the $\alpha_{kj}^{(l)}$'s. 
 Under the parametrization \eqref{par_approx}-\eqref{stuct_depend}, the parameters $\{\bm\xi^{(l)},l=1,...,L-1\}$ and $\{\bm\Sigma^{(l)},l=1,...,L-1\}$ are added to the parameter vector $\bm\eta$. Then the reparametrization trick is also performed for these parameters using the default representations for MVN and LFMVN, available out of the box in \hyperlink{https://pytorch.org/docs/stable/distributions.html}{PyTorch probabilities}. 
 


\subsection{Optimization by stochastic gradient}

For simplicity, we assume here that there is no dispersion parameters, so the target distribution is $p(\bm\beta,\bm\gamma|\D)$. We can rewrite the ELBO~\eqref{ELBO} as
\begin{align}
\mathcal{L}_{VI}(\boldsymbol\eta)=& \sum_{\boldsymbol\gamma \in \Gamma}\int_{\beta_\gamma}q_{\boldsymbol\eta}(\boldsymbol\beta,\boldsymbol\gamma)
[\log p(\D|\boldsymbol\beta,\boldsymbol\gamma) -\log\tfrac{q_{\boldsymbol\eta}(\boldsymbol\beta,\boldsymbol\gamma)}{p(\boldsymbol\beta,\boldsymbol\gamma)}]\text{d}\boldsymbol\beta. \label{ELBO2}
\end{align}
Due to the huge computational cost in the computation of gradients when $\Gamma$ and $\D$ are large, stochastic gradient methods using  Monte Carlo estimates for obtaining unbiased estimates of the gradients have become the standard approach for variational inference in such situations. Both the reparametrization trick and minibatching~\citep{kingma2015variational,blundell2015weight} are further applied.

Another complication in our setting is the discrete nature of $\bm\gamma$. Following~\citet{gal2017concrete}, we relax the Bernoulli distribution~\eqref{stuct_approx} with the \emph{Concrete distribution}:
\begin{align}
    \tilde\gamma=\gamma_{tr}(\nu,\delta;\alpha)=\text{sigmoid}(({\text{logit}(\alpha)-\text{logit}(\nu))/\delta}), \quad \nu\sim\text{Unif}[0,1],\label{repar_gamma}
\end{align}
where $\delta$ is a tuning parameter, which is selected to take some small value. In the zero limit, $\tilde\gamma$ reduces to a Bernoulli$(\alpha)$ variable.
Combined with the reparametrization of the $\beta$'s,
\begin{align}
    \beta=\beta_{tr}(\varepsilon;\kappa,\tau)=\kappa+\tau\varepsilon, \quad \varepsilon \sim N(0,1)\label{repar_beta}
\end{align}
we define the following approximation to the
ELBO:
\begin{align}
\mathcal{L}_{VI}^{\delta}(\boldsymbol\eta):= \int_{\boldsymbol\nu}\int_{\bm\varepsilon}q_{\bm\nu,\bm\varepsilon}(\bm\nu,\bm\varepsilon)[&\log p(\D|\beta_{tr}(\bm\varepsilon,\bm\kappa,\bm\tau),\gamma_{tr}(\bm\nu,\bm\alpha,\delta)) -\nonumber\\
&
\log\tfrac{q_{\boldsymbol\eta}(\beta_{tr}(\bm\varepsilon,\bm\kappa,\bm\tau),
                               \gamma_{tr}(\bm\nu,\bm\alpha,\delta))}
          {p(\beta_{tr}(\bm\varepsilon,\bm\kappa,\bm\tau),
                               \gamma_{tr}(\bm\nu,\bm\alpha,\delta))}]\text{d}\boldsymbol\varepsilon\text{d}\bm\nu
\end{align}
where the transformations on vectors are performed elementwise.
Further, due to that $\text{d}\boldsymbol\varepsilon\text{d}\bm\nu$ does not depend on $\boldsymbol\eta$, we can change the order of integration and differentiation when taking the gradient of $\mathcal{L}_{VI}^{\delta}(\boldsymbol\eta)$:
\begin{align}
\nabla_{\bm\eta}\mathcal{L}_{VI}^{\delta}(\boldsymbol\eta)= \int_{\boldsymbol\nu}\int_{\bm\varepsilon}q_{\bm\nu,\bm\varepsilon}(\bm\nu,\bm\varepsilon)
\nabla_{\bm\eta}[&\log p(\D|\beta_{tr}(\bm\varepsilon,\bm\kappa,\bm\tau),\gamma_{tr}(\bm\nu,\bm\alpha,\delta) -\nonumber\\
&
\log\tfrac{q_{\boldsymbol\eta}(\beta_{tr}(\bm\varepsilon,\bm\kappa,\bm\tau),
                               \gamma_{tr}(\bm\nu,\bm\alpha,\delta))}
          {p(\beta_{tr}(\bm\varepsilon,\bm\kappa,\bm\tau),
                               \gamma_{tr}(\bm\nu,\bm\alpha,\delta))}]\text{d}\boldsymbol\varepsilon\text{d}\bm\nu.\label{eq:ELBO.grad}
\end{align}
An unbiased estimator of $\nabla_{\bm\eta}\mathcal{L}_{VI}^{\delta}(\boldsymbol\eta)$
is then given in Proposition~\ref{ltwo}.
\begin{lemma}\label{ltwo}
Assume for all  $m=1,...,M$
$\left(\boldsymbol\nu^{(m)},\boldsymbol\eta^{(m)}\right) \sim q_{\boldsymbol\nu,\boldsymbol\varepsilon}(\boldsymbol\nu,\boldsymbol\eta)$
and
$S$ is a random subset of indices $\{1,...,n\}$ of size N. Also, assume the observations to be conditionally independent.  Then, for any $\delta>0$, an
unbiased estimator for the gradient of ${\mathcal{L}}_{VI}^\delta(\boldsymbol\eta)$  is given by
\begin{align}
\widetilde{\nabla}_{\bm\eta} {\mathcal{L}}_{VI}^\delta(\boldsymbol\eta)
= \frac{1}{M}\sum_{m=1}^M\Big[&\frac{n}{N}\sum_{i\in S}\nabla_{\boldsymbol\eta}\log{p(\boldsymbol{y_i}|\boldsymbol{x_i},\beta_{tr}(\bm\varepsilon^{(m)},\bm\kappa,\bm\tau),\gamma_{tr}(\bm\nu^{(m)},\bm\alpha,\delta))}-\nonumber\\
&\nabla_{\boldsymbol\eta}\log \tfrac{q_{\boldsymbol\eta} (\beta_{tr}(\bm\varepsilon^{(m)},\bm\kappa,\bm\tau),\gamma_{tr}(\bm\nu^{(m)},\bm\alpha,\delta))}{p(\beta_{tr}(\bm\varepsilon^{(m)},\bm\kappa,\bm\tau),\gamma_{tr}(\bm\nu^{(m)},\bm\alpha,\delta))}\Big].\label{eq:grad.LVI}
\end{align}
  \end{lemma}
\begin{proof}
From~\eqref{eq:ELBO.grad} we have that
\begin{align*}
\frac{1}{M}\sum_{m=1}^M
\nabla_{\bm\eta}[&\log p(\D|\beta_{tr}(\bm\varepsilon^{(m)},\bm\kappa,\bm\tau),\gamma_{tr}(\bm\nu^{(m)},\bm\alpha,\delta) -
\log\tfrac{q_{\boldsymbol\eta}(\beta_{tr}(\bm\varepsilon^{(m)},,\bm\kappa,\bm\tau),
                               \gamma_{tr}(\bm\nu^{(m)},\bm\alpha,\delta))}
          {p(\beta_{tr}(\bm\varepsilon^{(m)},\bm\kappa,\bm\tau),
                               \gamma_{tr}(\bm\nu,^{(m)},\bm\alpha,\delta))}]
\end{align*}
is an unbiased estimate of the gradient. 
Further, since we assume the observations to be conditionally independent, we have
\begin{align*}
    \nabla_{\bm\eta}\log p(\D|\beta_{tr}(\bm\varepsilon,\bm\kappa,\bm\tau),\gamma_{tr}(\bm\nu,\bm\alpha,\delta))
    =\sum_{i=1}^n\nabla_{\bm\eta}\log p(\bm y_i|\bm x_i;\beta_{tr}(\bm\varepsilon,\bm\kappa,\bm\tau),\gamma_{tr}(\bm\nu,\bm\alpha,\delta)),
\end{align*}
for which an unbiased estimator can be constructed through a random subset, showing the result. 
\end{proof}

\begin{algorithm}[h]
   \caption{ Doubly stochastic variational inference step}
   \label{alg:dsvi}
\begin{algorithmic}
   \STATE \textbf{sample} $N$ indices uniformly from $\{1,...,n\}$ defining $S$;
   \FOR{$m$ {\bfseries in} $\{1,...,M\}$}
   \FOR {$(k,j,l)\in\mathcal{B}$}
   \STATE \textbf{sample} $\nu_{kl}^{(l)}\sim \text{Unif}[0,1]$ and $\varepsilon_{kj}^{(l)}\sim N(0,1)$ 
   \ENDFOR
   \ENDFOR
  \STATE \textbf{calculate} $\widetilde{\nabla}_{\bm\eta} {\mathcal{L}}_{VI}^\delta(\boldsymbol\eta)$ according to~\eqref{eq:grad.LVI}
    \STATE \textbf{update} ${\boldsymbol\eta} \leftarrow {\boldsymbol\eta} + \boldsymbol{A} \widetilde{\nabla}_{\bm\eta}{{\mathcal{L}}}_{VI}^\delta(\boldsymbol\eta)$ 
\end{algorithmic}
\end{algorithm}

Algorithm~\ref{alg:dsvi}  describes one iteration of a doubly stochastic variational inference approach where updating is performed on the parameters for the case of mean field assumption (for simplicity).  
The set $\mathcal{B}$ is the collection of all combinations $j,k,l$  in the network. The matrix of learning rates $\boldsymbol{A}$ will always be diagonal, allowing for different step sizes on the parameters involved.
Following \citet{blundell2015weight}, constraints of $\tau^{(l)}_{kj}$ are incorporated by means of the reparametrization
$\tau^{(l)}_{kj} =\log(1+\exp(\rho^{(l)}_{kj}))$ where  $\rho^{(l)}_{kj}\in\mathcal{R}$.
Typically, updating is performed over a full \emph{epoch}, in which case the observations are divided into $n/N$ subsets and updating is performed sequentially over all subsets.

In case of the dependence structure~ \eqref{stuct_depend}, $\bm\alpha$ is sampled instead while $\bm\xi^{(l)}$ and the components of $\bm\Sigma^{(l)}$ go into $\bm\eta$.
Constraints on $\alpha^{(l)}_{kj}$ are incorporated by means of the reparametrization $
\alpha^{(l)}_{kj} = ({1+\exp(-\omega^{(l)}_{kj})})^{-1}$ with $\omega^{(l)}_{kj}\in\mathcal{R}$

Note that in the suggested algorithm partial derivatives with respect to marginal inclusion probabilities, as well as mean and standard deviation terms of the weights can be calculated by the usual backpropagation algorithm on a neural network. This algorithm assumes a known dispersion parameter $\phi$, but can be easily generalized to include learning about $\phi$ as well.

\subsection{Prediction}\label{OtherInf}

Once the estimates $\widehat{\boldsymbol\eta}$  of the parameters $\boldsymbol\eta$  of the variational approximating distribution are obtained, we go back to the original discrete model for $\bm\gamma$ (setting $\delta= 0$). Then, there are several ways to proceed with predictive inference. We list these below.

\paragraph{Fully Bayesian model averaging}
In this case, define 
\begin{equation}
\hat{p}(\Delta|\D) = \frac{1}{R}\sum_{r=1}^{R}p(\Delta|\bm\beta^r,\bm\gamma^r)\label{pred.bma}
\end{equation}
where $(\bm\beta^r,\bm\gamma^r)\sim q_{\hat{\bm\eta}}(\bm\beta,\bm\gamma)$.
This procedure
takes uncertainty in both the model structure $\bm\gamma$ and the parameters $\bm\beta$ into account in a formal Bayesian setting.  
A bottleneck of this approach  is that we have to both sample from a huge approximate posterior distribution of parameters and models \emph{and} keep all of the components of $\widehat{\boldsymbol\eta}$ stored during the prediction phase, which might be computationally and memory inefficient.

\paragraph{The posterior mean based model \citep{wasserman2000bayesian}} In this case we put $\beta^{(l)}_{kj} = \hat E\{\beta^{(l)}_{kj}|\D\}$ where
\begin{align*}
  E\{\beta^{(l)}_{kj}|\D\}
  =&p(\gamma^{(l)}_{kj}=1|\D)E\{\beta^{(l)}_{kj}|\gamma^{(l)}_{kj}=1,\D\}
  \approx\hat\alpha^{(l)}_{kj}\hat\kappa^{(l)}_{kj}.
\end{align*}
Here $\hat{\alpha}_{kj}^{(l)}$ is either the estimate of ${\alpha}_{kj}^{(l)}$ obtained through the variational inference procedure or, in case of dependence structure $\boldsymbol\alpha$ $E\{\alpha_{kj}^{(l)}|\mathcal{D}\}$. In the latter case, one formally would want to integrate out $\boldsymbol\alpha$ instead but this is not quite feasible in practice and extra sampling is avoided.
This approach specifies one dense model $\hat{\bm\gamma}$ with no sparsification. At the same time, no sampling is needed.

\paragraph{The median probability model \citep{barbieri2004optimal}}
This approach is based on the notion of a median probability model, which has been shown to be optimal in terms of predictions in the context of simple linear models. 
Here, we set $\gamma^{(l)}_{kj} = \text{I}(\hat\alpha^{(l)}_{kj}>0.5)$ while $\beta^{(l)}_{kj} \sim \gamma^{(l)}_{kj}N(\hat\kappa^{(l)}_{kj},{{\hat{\tau}^{2(l) }}_{kj}})$. A model averaging approach similar to~\eqref{pred.bma} is then applied.
Within this approach, we significantly sparsify the network and only sample from the distributions of those weights that have marginal inclusion probabilities above 0.5.

\paragraph{Median probability model-based inference combined with parameter posterior mean}
Here, again, we set $\gamma^{(l)}_{kj}=\text{I}(\hat\alpha^{(l)}_{kj}>0.5)$
but now we use $\beta_{kj}^{(l)}=\gamma^{(l)}_{kj}\hat\kappa^{(l)}_{kj}$. Similarly to the posterior mean-based model, no sampling is needed
but in addition, we only need to store the variational parameters of  $\widehat{\boldsymbol\eta}$ corresponding to marginal inclusion probabilities above 0.5. Hence, we significantly sparsify the BNN of interest and reduce the computational cost of the predictions drastically.

\paragraph{Post-training} Once it is decided to make inference based on a selected model, one might take several additional iterations of the training algorithm concerning the parameters of the models, having the architecture-related parameters fixed.  This  might give additional improvements in terms of the quality of inference as well as make the training steps much easier since the number of parameters is reduced dramatically. This is so since one does not have to estimate marginal inclusion probabilities $\boldsymbol\alpha$ 
any longer. Moreover, the number of weights $\beta_{jk}^{(l)}$'s corresponding to $\gamma_{jk}^{(l)}=1$ 
to make inference on is typically significantly reduced due to the sparsity induced by using the selected median probability model. It is also possible to keep the $\alpha_{jk}^{(l)}$'s fixed but still allow the  $\gamma_{jk}^{(l)}$'s to be random. 

\paragraph{Other model selecting criteria and alternative thresholding}

The median probability model is (at least in theory) not always feasible in the sense that one needs at least one connected path across all of the layers with all of the weights linking the neurons having marginal inclusion above 0.5. One way to resolve the issue is to use the most probable model (the model with the largest marginal posterior probability) instead of the median probability model. Then, conditionally on its configuration, one can sample from the distribution of the parameters, select the mean (mode) of the parameters, or post-train the distributions of the parameters. Other model selection criteria,  including DIC and WAIC,  can be used in the same way as the most probable model. Another heuristic way to tackle the issue is to replace conditioning on $\text{I}(\gamma^{(l)}_{kj}>0.5)$ with $\text{I}(\gamma^{(l)}_{kj}>\lambda)$, where $\lambda$ is a tuning parameter. The latter might also improve predictive performance in case too conservative priors on the model configurations are used. At the same time, we are not addressing the methods described in this paragraph in our experiments and rather leave them for further research. 

\section{Applications}\label{section4}

In-depth studies of the suggested variational approximations in the context of \emph{linear} regression models have been performed in earlier studies,  including multiple synthetic and real data examples with the aims of both recovering meaningful relations and predictions~\citep{carbonetto2012scalable,Hernandez2015}. The results from these studies show that the approximations based on the suggested variational family distributions are reasonably precise and indeed scalable, but can be biased. We will not address toy examples and simulation-based examples in this article and rather refer the curious readers to the very detailed and comprehensive studies in the references mentioned above, whilst we will address some more complex examples here. In particular, we will address the classification of MNIST \citep{lecun1998mnist} 
and fashion-MNIST \citep[FMNIST][]{xiao2017fashion} images as well as the PHONEME data \citep{hastie1995penalized}. Both MNIST and FMNIST datasets comprise of $70\,000$ grayscale images (size 28x28) from 10 categories (handwritten digits from 0 to 9, and "Top", "Trouser", "Pullover", "Dress", "Coat", "Sandal", "Shirt", "Sneaker", "Bag", and "Ankle Boot" Zalando's fashion items respectively), with $7\,000$ images per category. The training sets consist of $60\,000$ images, and the test sets have $10\,000$ images. For the PHONEME dataset, we have 256 covariates and 5 classes in the responses. In this dataset, we have $3\,500$ observations in the training set and $1\,000$ - in the test set. The PHONEME data are extracted from the TIMIT database (TIMIT Acoustic-Phonetic Continuous Speech Corpus, NTIS, US Dept of Commerce), which is a widely used resource for research in speech recognition. This dataset was formed by selecting five phonemes for
classification based on a digitized speech from this database.  The
phonemes are transcribed as follows: "sh" as in "she", "dcl" as in
"dark", "iy" as the vowel in "she", "aa" as the vowel in "dark", and
"ao" as the first vowel in "water".  

\paragraph{Experimental design} For all the  datasets, we address a dense neural network with the ReLU activation function, and multinomially distributed observations. For the two first examples, we have  10 classes and 784 input explanatory variables (pixels) while for the third one, we have 256 input variables and 5 classes. In all three cases, the network has 2 hidden layers with 400, and 600 neurons correspondingly. Priors for the parameters and model indicators were chosen according  to~\eqref{eq:prior.beta} with parameters in the priors specified through an empirical Bayes approach.
The inference was performed using the suggested doubly stochastic variational inference approach (Algorithm~\ref{alg:dsvi}) on 250 epochs with a batch size of 100.  $M$ was set to $1$ to reduce computational costs and due to the fact that this choice of $M$ is argued to be sufficient in combination with the reparametrization trick  \citep{Gal2016Uncertainty}. Up to  20  first epochs were used for pre-training of the models and parameters as well as empirically (aka Empirical Bayes) estimating the hyperparameters of the priors ($a_\psi,b_\psi,a_\beta,b_\beta$) through adding them into the computational graph. After that, the main training cycle began (with fixed hyperparameters on the priors).  We used the ADAM stochastic gradient ascent optimization \citep{kingma2014adam} with the diagonal matrix $\boldsymbol A$ in Algorithm~\ref{alg:dsvi} and the diagonal elements specified in Tables~\ref{tA} and S-1 for pre-training, and the main training stage. Typically, one would maximize the marginal likelihood for the Empirical Bayes methods, but since we do not have it available for the addressed models, its lower bound (ELBO) was used in the pre-training stage. After 250 training epochs, post-training was performed. When post-training the parameters, either with fixed marginal inclusion probabilities or with the median probability model, we ran additional $50$ epochs of the optimization routine 
 with $\boldsymbol A$ specified in the bottom rows of Tables~\ref{tA} and S-1. For the fully Bayesian model averaging approach, we used  both $R=1$ and $R=10$. Even though $R=1$ can give a poor Monte Carlo estimate of the prediction distribution, it can be of interest due to high sparsification. All the PyTorch implementations used in the experiments are available  
 \href{https://github.com/aliaksah/Variational-Inference-for-Bayesian-Neural-Networks-under-Model-and-Parameter-Uncertainty}{in our GitHub repository}.
 
\begin{table}[t]
 \small
  \centering
   \caption{\small Specifications of diagonal elements of $\boldsymbol A$ matrices for the step sizes of optimization routines for LBBNN-GP-MF and LBBNN-GP-MVN, see Table~\ref{te} for explanation of the abbreviations. Note that $A_{\omega}$ is only used in LBBNN-GP-MF, while $A_\xi$ and $A_\Sigma$ are only used in LBBNN-GP-MVN. For tuning parameters of LBBNN-GP-LFMVN, see Table~S-1 in the supplementary materials to the paper.}
  \begin{tabular}{l@{\hspace{0.4cm}}c@{\hspace{0.4cm}}ccccccc}
   \toprule
            &$A_\beta$, $A_\rho$&$A_\xi$&$A_{\omega}$&$A_\Sigma$&$A_{a_\psi}$,$A_{b_\psi}$&$A_{a_\beta}$,$A_{b_\beta}$\\
            \hline
Pre-training&0.00010&0.10000&0.10000&0.10000&0.00100&0.00001\\
Training    &0.00010&0.01000&0.00010&0.00010&0.00000&0.00000
\\
Post-training&0.00010&0.00000&0.00000&0.00000&0.00000&0.00000\\
\bottomrule
  \end{tabular}

 \label{tA}
 
\end{table}

\newcounter{num}
\setcounter{num}{10}
\ifodd\value{num} 
\begin{minipage}[b]{1\linewidth}\centering
\begin{minipage}[b]{0.7\linewidth}
\vspace{0.1cm}
\begin{table}[H]
 \caption{\small Inference possibilities. \textbf{SM} is a single sample, \textbf{MA} - model (sample) averaging, \textbf{MN} - posterior mean based inference, \textbf{MED} - selecting the median probability model, \textbf{WAIC}, \textbf{DIC}, \textbf{FIC} - selecting w.r.t. the corresponding criterion, \textbf{ADHOC} - ad hoc model selection, \textbf{PT} - post-training of the parameter, \textbf{PE} - point estimates of the predictions, \textbf{CI} - credible intervals for the predictions.}
  \label{t0}
  \small
  \centering
{
  \begin{tabular}{l@{\hspace{1.0cm}}c@{\hspace{0.7cm}}c@{\hspace{0.7cm}}c@{\hspace{0.7cm}}c@{\hspace{0.7cm}}c@{\hspace{0.7cm}}c}
   \toprule
Method&Full BNN&Gauss.&Mixt.&Concr.&Hors.&Studied\\\hline
&&\multicolumn{4}{l}{\hspace{0mm} \textbf{Dense}}\\\hline
$\text{SM}$&Joint&Par&Par&Joint&Par&\textbf{Yes}\\
$\text{MA}$&Joint&Par&Par&Joint&Par&\textbf{Yes}\\
$\text{MN}$&Joint&Par&Par&Joint&Par&\textbf{Yes}\\
\hline
&&\multicolumn{4}{l}{\hspace{0mm} \textbf{Model selection}}\\\hline
$\text{MED}$&+&-&-&-&-&\textbf{Yes}\\
$\text{WAIC, DIC, FIC}$&+&-&-&-&-&\textbf{No}\\
$\text{PRUNE}$&?&?&?&?&+&\textbf{Yes}\\
\hline
&&\multicolumn{4}{l}{\hspace{0mm} \textbf{Post training}}\\\hline
$\text{PT}$&MA/MS&-&-&MA&MS&\textbf{Yes}\\
\hline
&&\multicolumn{4}{l}{\hspace{0mm} \textbf{Inference}}\\\hline
$\text{PE}$ (Acc. All)&Joint&Par&Par&Joint&Par&\textbf{Yes}\\
$\text{CI}$ (Acc. 95\%)&Joint&Par&Par&Joint&Par&\textbf{Yes}\\
\bottomrule
  \end{tabular}
  }
\end{table}
\end{minipage}
\hspace{0.5cm}
\begin{minipage}[b]{0.25\linewidth}
\begin{table}[H]
 \caption{\small Medians and standard deviations of the average (per layer) marginal inclusion probability (see the text for definition) for our model for both MNIST and FMNIST data across 10 simulations.}
 \label{t3}
  \small
  \centering
  \begin{tabular}{l@{\hspace{0.4cm}}c@{\hspace{0.4cm}}c}
   \toprule
\multicolumn{3}{l}{\textbf{MNIST data}}\\
\hline
Layer& Med.  & SD.\\\hline
$\rho(\gamma^{(1)}|\D)$&0.0520&0.0005\\
$\rho(\gamma^{(2)}|\D)$&0.0598&0.0003\\
$\rho(\gamma^{(3)}|\D)$&0.2217&0.0064\\
\hline
\multicolumn{3}{l}{\textbf{FMNIST data}}\\
\hline
Layer& Med.  & SD.\\\hline
$\rho(\gamma^{(1)}|\D)$&0.0665&0.0004\\
$\rho(\gamma^{(1)}|\D)$&0.0613&0.0005\\
$\rho(\gamma^{(1)}|\D)$&0.2013&0.0051\\
\bottomrule
  \end{tabular}
\end{table}
\end{minipage}
\vspace{0.1cm}
 \end{minipage}
 \fi
We report results for \textit{our model} \textbf{LBBNN} applied with the Gaussian priors  (\textbf{GP}) for the slab components of $\beta$'s combined with variational inference based on Mean-Field (\textbf{MF}), MVN (\textbf{MVN}) and Low Factor MVN (for \textbf{LFMVN}, the predictions' results are reported in the supplemental material) dependence structures between the latent indicators. We use the combined names \textbf{LBBNN-GP-MF}, \textbf{LBBNN-GP-MVN}, and \textbf{LBBNN-GP-LFMVN} respectively to denote combination of model, prior and variational distribution.  In Section 4 of the supplementary materials to the paper, results from our early preprint \citep{hubin2019combining} for MNIST and FMNIST datasets where the hyperparameters are fixed ~\citep[corresponding to the setting also considered by][]{bai2020efficient}
are reported.

\begin{table}[!h]
 \caption{\small Abbreviations and evaluation metrics used when reporting the results of the experiments.}
  \label{te}
  \footnotesize
  \centering
\begin{tabular}{l|l}\hline
\textbf{Model}&\textbf{ Meaning}\\ \hline
BNN & Bayesian neural network\\
LBBNN & Latent binary Bayesian neural network\\
\hline \textbf{Parameters prior} &\textbf{ Meaning} \\ \hline 
GP&Independent Gaussian priors for weights\\
MGP&Independent mixture of Gaussians prior for  weights\\
HP&Independent horseshoe priors for  weights\\
\hline
\textbf{Inference} & \textbf{Meaning} \\ \hline
MF&Mean-field variational inference\\
MVN &Multivariate Gaussian structure for the inclusion probabilities\\ 
LFMVN &Low factor for the covariance of MVN structure for the inclusion probabilities\\ 
\hline $\bm \gamma$ &\textbf{Meaning} \\ \hline
SIM & Inclusion of the weights is drawn from the posterior of inclusion indicators\\
ALL & All weights are used\\
MED & Weights corresponding to the median probability model are used\\
PRN & Not pruned using a threshold-based rule weights are used\\
\hline $\bm \beta$& \textbf{Meaning}\\
  \hline
SIM & The included weights are drawn from their posterior\\
MEA & Posterior means of the weights are used\\
   \hline
\textbf{$R$} & \textbf{Meaning} \\    \hline
 10 & 10 samples are drawn\\
 1 & 1 sample is drawn or posterior means are used\\\hline
 \textbf{Evaluation metric} & \textbf{Meaning}\\\hline All cl Acc & Accuracy computed for all samples in the test set\\
 0.95 threshold Acc& Accuracy computed for those samples in the test set where the maximum \\& (across classes)  model averaged predictive posterior exceeds $0.95$\\
 0.95 threshold & Number of samples in the test set where the maximum \\Num.cl& (across classes)  model averaged predictive posterior exceeds $0.95$\\
 Dens. level & Fraction of weights that are used to make predictions\\
 Epo. time & Average time elapsed per epoch of training\\
   \hline
\end{tabular}
\end{table}

In addition, we also used several relevant \textit{baselines}. In particular, we addressed a standard Dense BNN with Gaussian priors and mean-field variational inference~\citep{graves2011practical}, denoted as \textbf{BNN-GP-MF}, which can be seen as a special case of our original model with all $\gamma^{(l)}_{kj}$ being fixed and equal to 1 and no prior on the variance components of weights. This model is important in measuring how predictive power is changed due to introducing sparsity.
Furthermore, we report the results for a Dense BNN with mixture priors (\textbf{BNN-MGP-MF}) with two Gaussian components of the mixtures \citep{blundell2015weight} with probabilities of 0.5 for each and variances equal to 1 and $e^{-6}$ correspondingly. Additionally, we have addressed two popular sparsity-inducing approaches, in particular, a dense network  with Concrete dropout (\textbf{BNN-GP-CMF}) \citep{gal2017concrete} and a dense network  with Horseshoe priors (\textbf{BNN-HP-MF}) \citep{louizos2017bayesian}. Finally, a frequentist fully connected  neural network (\textbf{FNN}) (with posthoc weight pruning) was used as a more basic baseline. We only report the results for FNN in the supplementary materials to make the experimental design cleaner.  All of the baseline methods (including the FNN) also have 2 hidden layers with 400, and 600 neurons correspondingly. They were trained for 250 epochs with an Adam optimizer (with a learning rate $a=0.0001$ for all involved parameters) and a batch size equal to 100. 
For the BNN with Horseshoe priors, we are reporting statistics separately before and after ad-hoc pruning (\text{PRN}) of the weights. Post-training (when necessary) was performed for additional 50 epochs. For FNN, for all three experiments, we performed weight and neuron pruning \citep{blalock2020state} to have the same sparsity levels as those obtained by the Bayesian approaches to make them directly comparable. Pruning of FNN was based on removing the corresponding share of weights/neurons having the smallest magnitude (absolute value). No uncertainty was taken into consideration and neither was structure learning considered for FNNs.

For prediction, several methods were described in Section~\ref{OtherInf}. All essentially boils down to choices on how to treat the model parameters $\bm\gamma$ and the weights $\bm\beta$. For $\bm\gamma$, we can either simulate (\textbf{SIM}) from the (approximate) posterior or use the median 
probability model (\textbf{MED}). An alternative for \textbf{BNN-HP-MF} here is the pruning method (\textbf{PRN}) applied in  \citet{louizos2017bayesian} . We also consider the choice of including all weights for some of the baseline methods (\textbf{ALL}).
For $\bm\beta$, we consider either sampling from the (approximate) posterior (\textbf{SIM}) or using the posterior mean (\textbf{MEA}). Under this notation, the fully Bayesian model averaging from Section~\ref{OtherInf} is denoted as \textbf{SIM SIM}, whilst {the posterior mean based model} as \textbf{ALL MEA}, {the median probability model} as \textbf{MED SIM}, the and {the median probability model combined with parameter posterior mean} as \textbf{MED MEA}.

We then evaluated accuracies (\textbf{Acc} - the proportion of the correctly classified images). 
Accuracies based on the median probability model (through either $R=1$ or $R=10$) and the posterior mean models were also obtained. Finally, accuracies based on post-training of the parameters with fixed marginal inclusion probabilities and post-training of the median probability model were evaluated. For the cases when model averaging is addressed ($R=10$), we are additionally reporting accuracies when classification is only performed if the maximum model-averaged class probability exceeds 95\%  as suggested by \citet{posch2019variational}. Otherwise, a doubt decision is made~\citep[][sec 2.1]{ripley2007pattern}. In this case, we both
report the accuracy within the classified images  as well as the number of classified images.
Finally,  we are reporting the overall density level (the fraction of $\gamma_{kj}^{(l)}$'s equal to one within at least one of the simulations), for different approaches. 
To guarantee reproducibility, summaries (medians, minimums, maximums) across 10 independent runs of the described experiment $s \in \{1,...,10\}$ were computed for all of these statistics. 
Estimates of the marginal inclusion probabilities~$\hat{p}(\gamma_{kj}^{(l)}=1|\D)$ based on the suggested variational approximations were also computed for all of the weights. In order to compress the presentation of the results, we only present the mean marginal inclusion probabilities for each layer $l$ as $\rho(\gamma^{(l)}|\D) := \frac{1}{p^{(l+1)}p^{(l)}}\sum_{kj}\hat{p}(\gamma_{kj}^{(l)}=1|\D)$, summarized in Table~\ref{t3}. To make the abbreviations used in the reported results more clear, we provide Table~\ref{te} with their short summaries.

 \paragraph{MNIST} The results reported in Tables~\ref{t1} and~\ref{t3} (with some additional results on {LBBNN-GP-LFMVN} and post-training reported in Tables~S-2 and~S-8 in the supplementary material) show that within our LBBNN approach: a) model averaging across different BNNs ($R=10$) gives significantly higher accuracy than the accuracy of a random individual BNN from the model space ($R=1$); b) the median probability model and posterior mean based model also perform significantly better than a randomly sampled model. The performance of  the median probability model and posterior mean-based model is in fact on par with full model averaging; c) according to Table~\ref{t3} and Figure~\ref{Fig:hist1}, for the mean-field variational distribution, the majority of the weights of the models have very low marginal inclusion probabilities for the weights at layers 1 and 2, while more weights have high marginal inclusion probabilities at layer 3 (although a significant reduction also at this layer). This resembles the structure of convolutional neural networks (CNN) where typically one first has a set of sparse convolutional layers, followed by a few fully connected layers. Unlike CNNs the structure of sparsification is learned automatically within our approach; d) for the MVN with full rank structure within variational approximation, the input layer is the most dense, followed by extreme sparsification in the second layer and a moderate sparsification at layer 3; e) the MVN approach with a low factor parametrization of the covariance matrix (results in the supplementary) only provides very moderate sparsification not exceeding 50\% of the weight parameters; f) variations of all of the performance metrics across simulations are low, showing stable behavior across the repeated experiments; 
 g) inference with a doubt option gives almost perfect accuracy, however, this comes at a price of rejecting to classify some of the items.

\begin{table}[!htbp]
 \caption{\small Performance metrics  for the MNIST data  for the compared approaches.
All results are medians across $10$ repeated experiments (with min and max included in parentheses). 
No post-training is used. For further details see Table \ref{te}.}
  \label{t1}
  \footnotesize
  \centering
\begin{tabular}{llll|ccccc}
   \hline
\multicolumn{2}{c}{Prediction}&Model-Prior-&&\multicolumn{1}{c}{All cl}&\multicolumn{2}{c}{0.95 threshold}&Dens.&Epo.\\
$\bm\gamma$&$\bm\beta$&Method&$R$& Acc& Acc&Num.cl&level&time\\
    \hline
SIM&SIM&LBBNN-GP-MF&1&0.968 (0.966,0.970)&-&-&0.090&8.363\\
SIM&SIM&LBBNN-GP-MF&10&0.981 (0.979,0.982)&0.999 &8322&1.000&8.363\\
ALL&MEA&LBBNN-GP-MF&1&0.981 (0.980,0.983)&-&-&1.000&8.363\\ 
MED&SIM&LBBNN-GP-MF&1&0.969 (0.968,0.974)&-&-&0.079&8.363\\ 
MED&SIM&LBBNN-GP-MF&10&0.980 (0.979,0.982)&0.999 &8444&0.079&8.363\\ 
MED&MEA&LBBNN-GP-MF&1&0.981 (0.980,0.983)&-&-&0.079&8.363\\
\hline
SIM&SIM&LBBNN-GP-MVN&1&0.965 (0.964,0.966)&-&-&0.180&9.651\\
SIM&SIM&LBBNN-GP-MVN&10&0.978 (0.976,0.979)&1.000 &7818&1.000&9.651\\
ALL&MEA&LBBNN-GP-MVN&1&0.978 (0.976,0.980)&-&-&1.000&9.651\\ 
MED&SIM&LBBNN-GP-MVN&1&0.968 (0.966,0.969)&-&-&0.163&9.651\\ 
MED&SIM&LBBNN-GP-MVN&10&0.977 (0.975,0.979)&1.000 &7928&0.163&9.651\\ 
MED&MEA&LBBNN-GP-MVN&1&0.974 (0.972,0.976)&-&-&0.163&9.651\\
\hline
ALL&SIM&BNN-GP-MF&1&0.965 (0.965,0.966)&-&-&1.000&5.094\\
ALL&SIM&BNN-GP-MF&10&0.984 (0.982,0.985)&0.999 &8477&1.000&5.094\\
ALL&MEA&BNN-GP-MF&1&0.984 (0.982,0.985)&-&-&1.000&5.094\\
\hline
ALL&SIM&BNN-MGP-MF&1&0.965 (0.964,0.967)&-&-&1.000&5.422\\
ALL&SIM&BNN-MGP-MF&10&0.982 (0.981,0.983)&0.999 &8329&1.000&5.422\\
ALL&MEA&BNN-MGP-MF&1&0.983 (0.981,0.984)&-&-&1.000&5.422\\
\hline
SIM&SIM&BNN-GP-CMF&1&0.982 (0.894,0.984)&-&-&0.226&3.477\\
SIM&SIM&BNN-GP-CMF&10&0.984 (0.896,0.986)&0.995 &9581&1.000&3.477\\
ALL&MEA&BNN-GP-CMF&1&0.984 (0.893,0.986)&-&-&1.000&3.477\\
\hline
SIM&SIM&BNN-HP-MF&1&0.964 (0.962,0.967)&-&-&1.000&4.254\\
SIM&SIM&BNN-HP-MF&10&0.982 (0.981,0.983)&1.000&0003&1.000&4.254\\
ALL&MEA&BNN-HP-MF&1&0.966 (0.963,0.968)&-&-&1.000&4.254\\
PRN&SIM&BNN-HP-MF&1&0.965 (0.962,0.969)&-&-&0.194&4.254\\
PRN&SIM&BNN-HP-MF&10&0.982 (0.981,0.983)&1.000 &0002&0.194&4.254\\
PRN&MEA&BNN-HP-MF&1&0.965 (0.963,0.968)&-&-&0.194&4.254\\
\bottomrule
  \end{tabular}
\end{table}

For other approaches, it is also the case that h) both using the posterior mean-based model and using sample averaging improves accuracy compared to a single sample from the parameter space; i) variability in the estimates of the target parameters is low for the dense BNNs with Gaussian/mixture of Gaussians priors and BNN with horseshoe priors and rather high for the Concrete dropout approach. When it comes to comparing our approach to baselines we notice that j) dense approaches outperform sparse approaches in terms of the accuracy in general; k) Concrete dropout marginally outperforms other approaches in terms of median accuracy, however, it exhibits large variance, whilst our full BNN and the compressed BNN with horseshoe priors yield stable performance across experiments; l) neither our approach nor baselines managed to reach state of the art results in terms of hard classification accuracy of predictions \citep{palvanov2018comparisons}; m) including a 95\% threshold for making a classification results in a very low number of classified cases for the horseshoe priors (it is extremely underconfident),  the Concrete dropout approach seems to be overconfident when doing inference with the doubt option (resulting in lower accuracy but a larger number of decisions), the full BNN, and BNN with Gaussian and mixture of Gaussian priors give less classified cases than the Concrete dropout approach but reach significantly higher accuracy; n) this might mean that the thresholds need to be calibrated towards the specific methods; o) our approach under the mean-field variational approximation and the full rank MVN structure of variational approximation yields the highest sparsity of weights when using the median probability model. Also, q) post-training (results in the supplementary) does not seem to significantly improve either the predictive quality of the models or uncertainty handling; o) all BNN for all considered sparsity levels on a given configuration of the network depth and widths are significantly outperforming the frequentist counterpart (with the corresponding same sparsity levels) in terms of the generalization error. Finally, in terms of computational time, r) as expected FNNs were the fastest in terms of time per epoch, whilst for the Bayesian approaches we see a strong positive correlation between the number of parameters and computational time, where BNN-GP-CMF is the fastest method and LBBNN-GP-MVN is the slowest. All times were obtained whilst training our models on a GeForce RTX 2080 Ti GPU card. Having said that, it is important to notice that the speed difference between the fastest and slowest Bayesian approach is less than 3 times, which given the fact that the time is also influenced by the implementation of different methods and a potentially different load of the server when running the experiments might be considered quite a tolerable difference in practice.

\paragraph{FMNIST}\label{FMNIST}
The same set of approaches, model specifications, and tuning parameters of the algorithms as in the MNIST example were used for this application. The results a)- r) for FMNIST data, based on Tables~\ref{t11}, ~\ref{t3}, and Tables~S-3 and~S-9 in the supplementary meterials as well as in Figure~\ref{Fig:hist1} are completely consistent with the results from the MNIST experiment, however, the predictive performances for all of the approaches are poorer on FMNIST. Also whilst full BNN and BNN with horseshoe priors on FMNIST get lower sparsity levels than on MNIST, Concrete dropout here improves in this sense compared to the previous example. For FNN, the same conclusions as those obtained for the MNIST data set are valid.

\begin{table}[!htbp]
 \caption{\small Performance metrics  for the FMNIST data for the suggested in the article Bayesian approaches to BNN. For further details, see Table~\ref{te} and the caption of Table~\ref{t1}.}
  \label{t11}
  \footnotesize
  \centering
\begin{tabular}{llll|ccccc}
   \hline
\multicolumn{2}{c}{Prediction}&Model-Prior-&&\multicolumn{1}{c}{All cl}&\multicolumn{2}{c}{0.95 threshold}&Dens.&Epo.\\
$\bm\gamma$&$\bm\beta$&Method&$R$& Acc& Acc&Num.cl&level&time\\
    \hline
SIM&SIM&LBBNN-GP-MF&1&0.864 (0.861,0.866)&-&-&0.120&7.969\\
SIM&SIM&LBBNN-GP-MF&10&0.883 (0.881,0.886)&0.995 &4946&1.000&7.969\\
ALL&MEA&LBBNN-GP-MF&1&0.882 (0.879,0.887)&-&-&1.000&7.969\\
MED&SIM&LBBNN-GP-MF&1&0.867 (0.864,0.871)&-&-&0.108&7.969\\
MED&SIM&LBBNN-GP-MF&10&0.883 (0.880,0.886)&0.995 &5025&0.108&7.969\\
MED&MEA&LBBNN-GP-MF&1&0.880 (0.877,0.886)&-&-&0.108&7.969\\
\hline
SIM&SIM&LBBNN-GP-MVN&1&0.858 (0.854,0.859)&-&-&0.156&9.504\\
SIM&SIM&LBBNN-GP-MVN&10&0.879 (0.874,0.880)&0.995 &4503&1.000&9.504\\
ALL&MEA&LBBNN-GP-MVN&1&0.875 (0.873,0.876)&-&-&1.000&9.504\\ 
MED&SIM&LBBNN-GP-MVN&1&0.865 (0.860,0.866)&-&-&0.129&9.504\\ 
MED&SIM&LBBNN-GP-MVN&10&0.877 (0.875,0.879)&0.995 &4694&0.129&9.504\\ 
MED&MEA&LBBNN-GP-MVN&1&0.871 (0.868,0.875)&-&-&0.129&9.504\\
    \hline
ALL&SIM&BNN-GP-MF&1&0.864 (0.863,0.866)&-&-&1.000&5.368\\
ALL&SIM&BNN-GP-MF&10&0.893 (0.890,0.894)&0.997 &5089&1.000&5.368\\
ALL&MEA&BNN-GP-MF&1&0.886 (0.882,0.888)&-&-&1.000&5.368\\
\hline
ALL&SIM&BNN-MGP-MF&1&0.867 (0.866,0.868)&-&-&1.000&4.803\\
ALL&SIM&BNN-MGP-MF&10&0.893 (0.892,0.897)&0.996 &5151&1.000&4.803\\
ALL&MEA&BNN-MGP-MF&1&0.888 (0.885,0.890)&-&-&1.000&4.803\\
\hline
SIM&SIM&BNN-GP-CMF&1&0.896 (0.820,0.902)&-&-&0.094&3.369\\
SIM&SIM&BNN-GP-CMF&10&0.897 (0.823,0.901)&0.942 &8825&1.000&3.369\\
ALL&MEA&BNN-GP-CMF&1&0.896 (0.821,0.901)&-&-&1.000&3.369\\
\hline
SIM&SIM&BNN-HP-MF&1&0.864 (0.863,0.869)&-&-&1.000&4.613\\
SIM&SIM&BNN-HP-MF&10&0.887 (0.886,0.889)&1.000 &0181&1.000&4.613\\
ALL&MEA&BNN-HP-MF&1&0.867 (0.861,0.868)&-&-&1.000&4.613\\
PRN&SIM&BNN-HP-MF&1&0.865 (0.860,0.868)&-&-&0.302&4.613\\
PRN&SIM&BNN-HP-MF&10&0.887 (0.884,0.888)&1.000 &0179&0.302&4.613\\
PRN&MEA&BNN-HP-MF&1&0.865 (0.862,0.869)&-&-&0.302&4.613\\
\bottomrule
  \end{tabular}
\end{table}


\paragraph{PHONEME}\label{sounds}

Finally, the same set of approaches, model specifications (except for having 256 input covariates and 5 classes  of the responses), and tuning parameters of the algorithms as in the MNIST and FMNIST examples were used for the classification of PHONEME data. The results a)- r) for the PHONEME data, based on Tables~\ref{t111},~\ref{t3}, and Tables~S-4 and~S-10 in the supplementary are also overall consistent with the results from the MNIST and FMNIST experiments, however, predictive performances for all of the approaches are better than on FMNIST yet poorer than on MNIST. All of the methods, where sparsifications are possible, gave a lower sparsity level for this example. Yet, rather considerable sparsification is still shown to be feasible. For FNN, the same conclusions as those obtained for MNIST and FMNIST data sets are valid, though the deterioration of performance of FNN here was less drastic. Also, as we demonstrate in Figures S-5 - S-8 in the supplementary materials, the conclusions are consistent across various width configurations of Bayesian neural networks. Also, sparsification increases with increased width for all the methods, yet the growth in sparsity is not proportional to the growth of width. 

\begin{table}[!t]
 \caption{\small Performance metrics  for the PHONEME data for the suggested in the article Bayesian approaches to BNN. For further detail see Table~\ref{te} and the caption of Table~\ref{t1}.}
  \label{t111}
  \footnotesize
  \centering
\begin{tabular}{llll|ccccc}
   \hline
\multicolumn{2}{c}{Prediction}&Model-Prior-&&\multicolumn{1}{c}{All cl}&\multicolumn{2}{c}{0.95 threshold}&Dens.&Epo.\\
$\bm\gamma$&$\bm\beta$&Method&$R$& Acc& Acc&Num.cl&level&time\\
    \hline
SIM&SIM&LBBNN-GP-MF&1&0.913 (0.898,0.929)&-&-&0.371&0.433\\
SIM&SIM&LBBNN-GP-MF&10&0.927 (0.923,0.933)&0.992 &690&1.000&0.433\\
ALL&MEA&LBBNN-GP-MF&1&0.925 (0.921,0.933)
&-&-&1.000&0.433\\
MED&SIM&LBBNN-GP-MF&1&0.923 (0.910,0.928)&-&-&0.307&0.433\\
MED&SIM&LBBNN-GP-MF&10&0.925 (0.912,0.934)&0.984&757&0.307&0.433\\
MED&MEA&LBBNN-GP-MF&1&0.925 (0.913,0.932)&-&-&0.307&0.433\\
\hline
SIM&SIM&LBBNN-GP-MVN&1&0.919 (0.911,0.927)&-&-&0.255&0.505\\
SIM&SIM&LBBNN-GP-MVN&10&0.929 (0.927,0.935)&0.995 &649&1.000&0.505\\
ALL&MEA&LBBNN-GP-MVN&1&0.926 (0.918,0.931)&-&-&1.000&0.505\\ 
MED&SIM&LBBNN-GP-MVN&1&0.925 (0.916,0.929)&-&-&0.225&0.505\\ 
MED&SIM&LBBNN-GP-MVN&10&0.929 (0.925,0.933)&0.995 &668&0.225&0.505\\ 
MED&MEA&LBBNN-GP-MVN&1&0.924 (0.921,0.928)&-&-&0.225&0.505\\
    \hline
ALL&SIM&BNN-GP-MF&1&0.915	(0.907,0.919)&-&-&1.000&0.203\\
ALL&SIM&BNN-GP-MF&10&0.919 (0.900,0.929)&0.966 &834&1.000&0.203\\
ALL&MEA&BNN-GP-MF&1&0.917	(0.901,0.922)&-&-&1.000&0.203\\
\hline
ALL&SIM&BNN-MGP-MF&1&0.913 (0.910,0.925)&-&-&1.000&0.208\\
ALL&SIM&BNN-MGP-MF&10&0.916 (0.912,0.926)&0.969&833&1.000&0.208\\
ALL&MEA&BNN-MGP-MF&1&0.921 (0.914,0.926)&-&-&1.000&0.208\\
\hline
SIM&SIM&BNN-GP-CMF&1&0.879 (0.706,0.906)&-&-&0.509&0.103\\
SIM&SIM&BNN-GP-CMF&10&0.922 (0.918,0.930)&0.965 &187&1.000&0.103\\
ALL&MEA&BNN-GP-CMF&1&0.873 (0.712,0.904)
&-&-&1.000&0.103\\
\hline
SIM&SIM&BNN-HP-MF&1&0.921 (0.915,0.929)&-&-&1.000&0.136\\
SIM&SIM&BNN-HP-MF&10& 0.921 (0.915,0.926)&0.895 &019&1.000&0.136\\
ALL&MEA&BNN-HP-MF&1&0.921 (0.916,0.926)&-&-&1.000&0.136\\
PRN&SIM&BNN-HP-MF&1&0.919 (0.909,0.926)&-&-&0.457&0.136\\
PRN&SIM&BNN-HP-MF&10&0.919 (0.916,0.927)&0.926 &028&0.457&0.136\\
PRN&MEA&BNN-HP-MF&1&0.920 (0.914,0.926)&-&-&0.457&0.136\\
\bottomrule
  \end{tabular}
\end{table}

\begin{table}[!htbp]
 \caption{\small Medians and standard deviations of the average (per layer) marginal inclusion probability (see the text for the definition) for our model for both MNIST and FMNIST data across 10 repeated experiments.}
 \label{t3}
  \small
  \centering
  \begin{tabular}{l@{\hspace{0.4cm}}c@{\hspace{0.4cm}}ccc}
   \toprule
&\multicolumn{1}{l}{\textbf{MNIST data}}&\multicolumn{1}{l}{\textbf{FMNIST data}}&\multicolumn{1}{l}{\textbf{PHONEME data}}\\
\hline
 \multicolumn{2}{l}{\textbf{LBBNN-GP-MF}}&\\\hline

$\rho(\gamma^{(1)}=1|\D)$&0.0844 (0.0835,0.0853)&0.1323 (0.1291,0.1349)& 0.3806 (0.3764,0.3838)\\
$\rho(\gamma^{(2)}=1|\D)$&0.0959 (0.0942,0.0967)&0.1005 (0.0981,0.1020)& 0.3670  (0.3641,0.3699) \\
$\rho(\gamma^{(3)}=1|\D)$&0.2945 (0.2808,0.3056)&0.2790 (0.2709,0.2921)& 0.4236   (0.4053,0.4367)\\
\hline
 \multicolumn{2}{l}{\textbf{LBBNN-GP-MVN}}&\\\hline

$\rho(\gamma^{(1)}=1|\D)$&0.2975 (0.2928,0.2993)&0.2461 (0.2410,0.2515)&0.3201 (0.3142,0.3273)\\
$\rho(\gamma^{(2)}=1|\D)$&0.0368 (0.0363,0.0377)&0.0392 (0.0383,0.0398)&0.2287 (0.2235,0.2355)\\
$\rho(\gamma^{(3)}=1|\D)$&0.1394 (0.1311,0.1475)&0.1462 (0.1368,0.1521)&0.2763 (0.2632,0.2953)\\
\hline
 \multicolumn{2}{l}{\textbf{LBBNN-GP-LFMVN}}&\\\hline

$\rho(\gamma^{(1)}=1|\D)$&0.4474 (0.4448,0.4498)&0.4589 (0.4565,0.4603)&0.4973 (0.4965,0.4987)\\
$\rho(\gamma^{(2)}=1|\D)$&0.4525 (0.4501,0.4537)&0.4516 (0.4493,0.4528)&0.4972 (0.4952,0.4990)\\
$\rho(\gamma^{(3)}=1|\D)$&0.4815 (0.4685,0.4871)&0.4805 (0.4654,0.4868)&0.4979 (0.4925,0.5048)\\
\hline
\bottomrule
  \end{tabular}

\end{table}

\begin{figure}[!htbp]
\begin{minipage}[t]{1.0\linewidth}
\begin{minipage}[t]{0.499\linewidth}
\centering
MNIST

\includegraphics[trim={0.1cm 0.5cm 1.5cm 1.2cm},clip, width=1\linewidth]{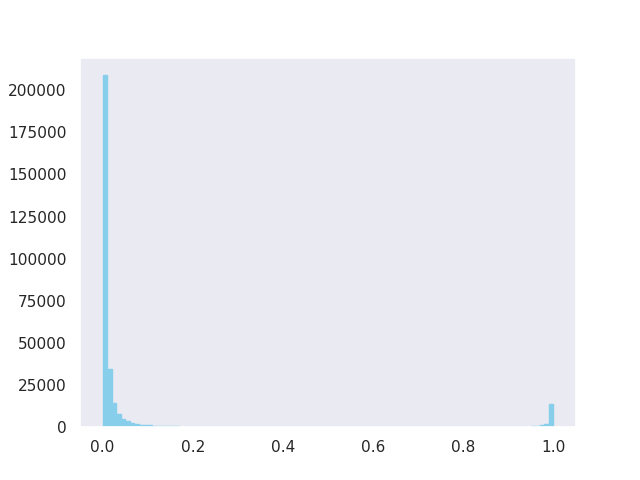}
\includegraphics[trim={0.1cm 0.5cm 1.5cm 1.2cm},clip,width=1\linewidth]{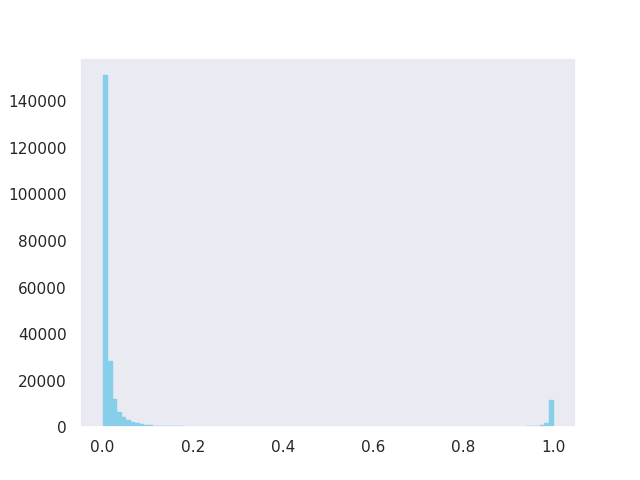}
\includegraphics[trim={0.1cm 0.5cm 1.5cm 1.2cm},clip,width=1\linewidth]{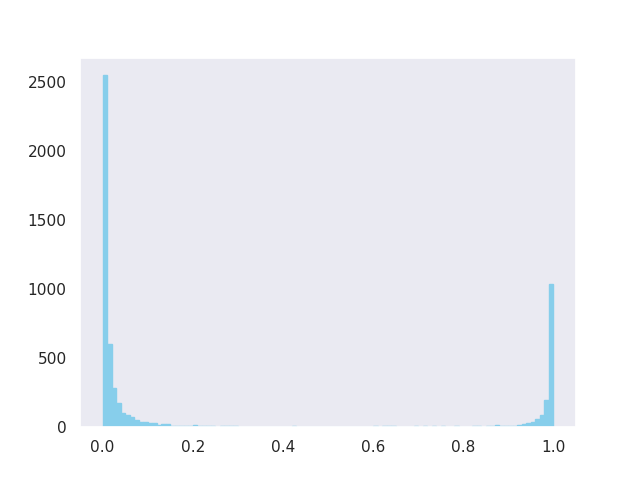}

\end{minipage}
\begin{minipage}[t]{0.499\linewidth}
\centering
FMNIST

\includegraphics[trim={0.1cm 0.5cm 1.5cm 1.2cm},clip,width=1\linewidth]{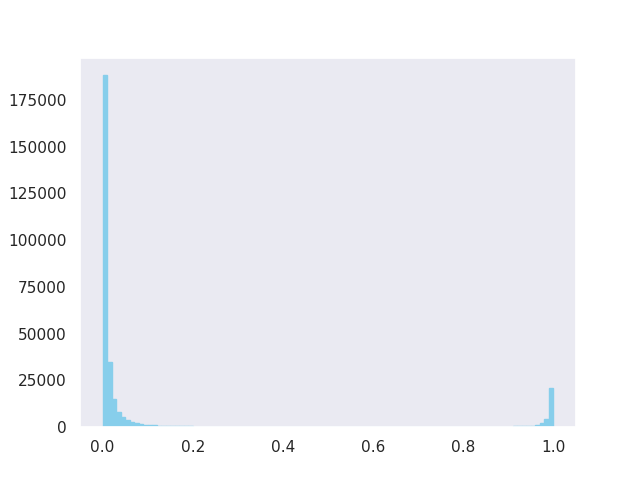}
\includegraphics[trim={0.1cm 0.5cm 1.5cm 1.2cm},clip,width=1\linewidth]{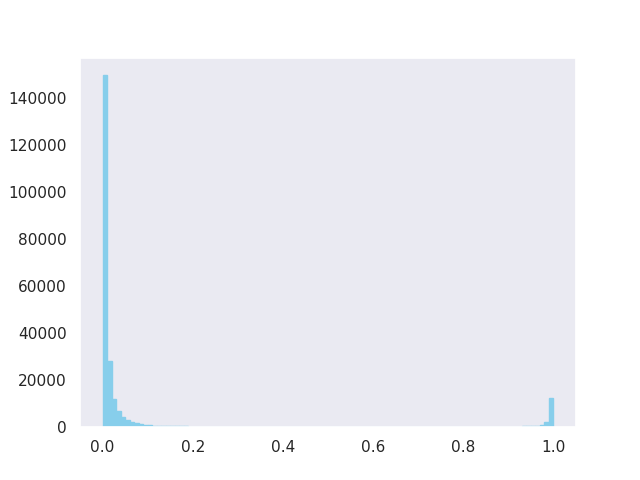}
\includegraphics[trim={0.1cm 0.5cm 1.5cm 1.2cm},clip,width=1\linewidth]{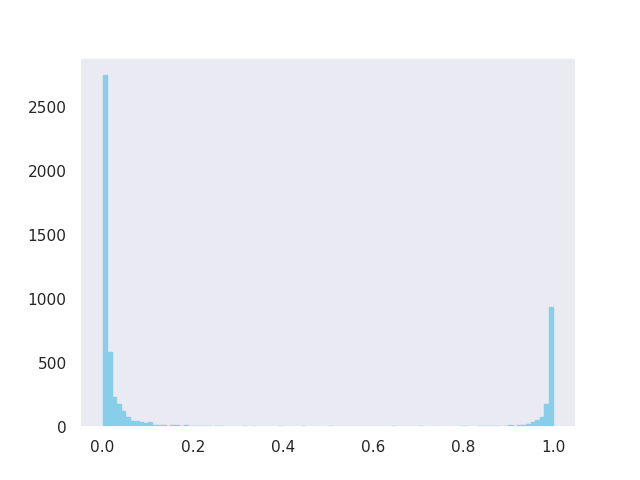}

\end{minipage}
\end{minipage}
\caption{An illustration of histograms of the marginal inclusion probabilities of the weights for the three layers (from top to bottom) of LBBNN-GP-MF from simulation $s=10$ for MNIST (left) and FMNIST (right).}\label{Fig:hist1}
\end{figure}

\paragraph{Out-of-domain experiments}
Following the example of measuring the in and out-of-domain uncertainty suggested in \citet{blogwu}, we will first look at the ability of the LBBNN-GP-MF approach to give confidence in its predictions by means of trying to classify a sample from FMNIST images with samples from the posterior predictive distribution based on the joint posterior of models and parameters trained on MNIST dataset and compare this to the results for a sample of images from the test set of MNIST data. The results are reported for the joint posterior (of models and parameters) obtained in experiment run $s = 10$. As can be seen in Figure~\ref{Fig:uncer}, the samples from LBBNN-GP-MF give highly confident predictions for the MNIST dataset with almost no variance in the samples from the posterior predictive distribution. At the same time, the out-of-domain uncertainty, related to the samples from the posterior predictive distribution based on FMNIST data, is typically high (with some exceptions) showing low confidence of the samples from the posterior predictive distribution in this case. The reversed example of inference on FMNIST and uncertainty related to MNIST data, illustrated in Figure~\ref{Fig:funcer}, leads to the same conclusions. More or less identical results were obtained for the LBBNN-GP-MVN and LBBNN-GP-LFMVN approaches, but they are not reported in the paper due to space constraints.

\begin{figure}[!htbp]
\centering
\includegraphics[trim={3.5cm 0.7cm 3.2cm 1.3cm},clip,width=0.49\linewidth]{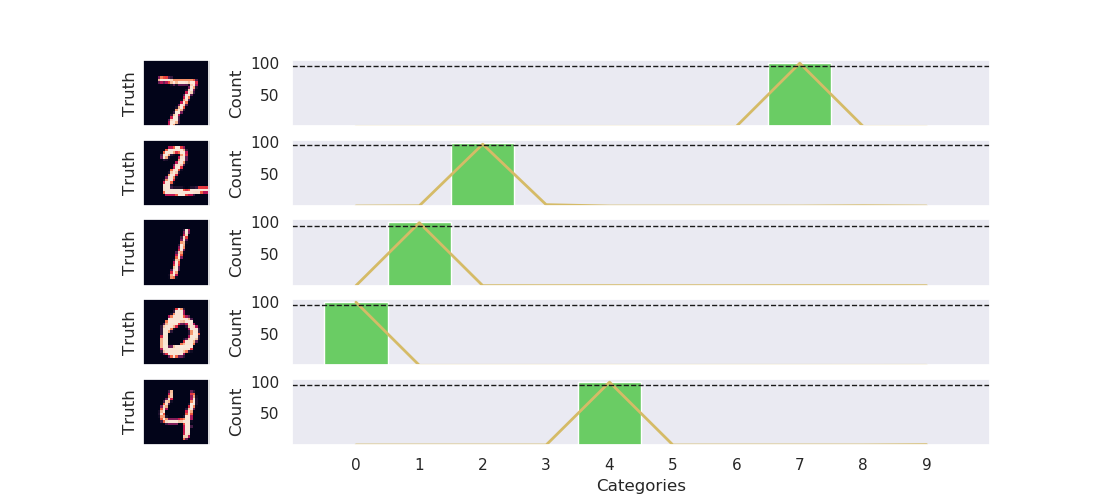}
\includegraphics[trim={3.5cm 0.7cm 3.2cm 1.3cm},clip,width=0.49\linewidth]{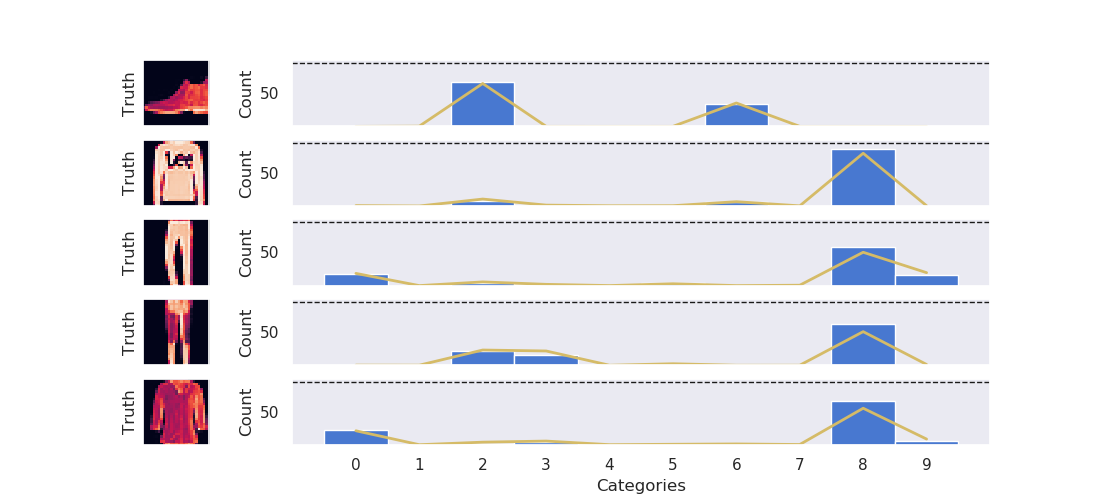}
\caption{Uncertainty related to the in-domain test data (MNIST, left) and out-of-domain test data (FMNIST, right) based on 100 samples from the posterior predictive distribution.  Yellow lines are model-averaged posterior class probabilities (in percent).
Green bars mark the correct classes, blue bars for other samples (with heights corresponding to an alternative estimate of class probabilities using hard classification within each of the replicates in the prediction procedure. The dashed black lines give the 95\% threshold for making decisions with doubt possibilities).
The original images are depicted to the left.}\label{Fig:uncer}
\end{figure}
\begin{figure}[!htbp]
\centering
\includegraphics[trim={3.5cm 0.7cm 3.2cm 1.3cm},clip,width=0.49\linewidth]{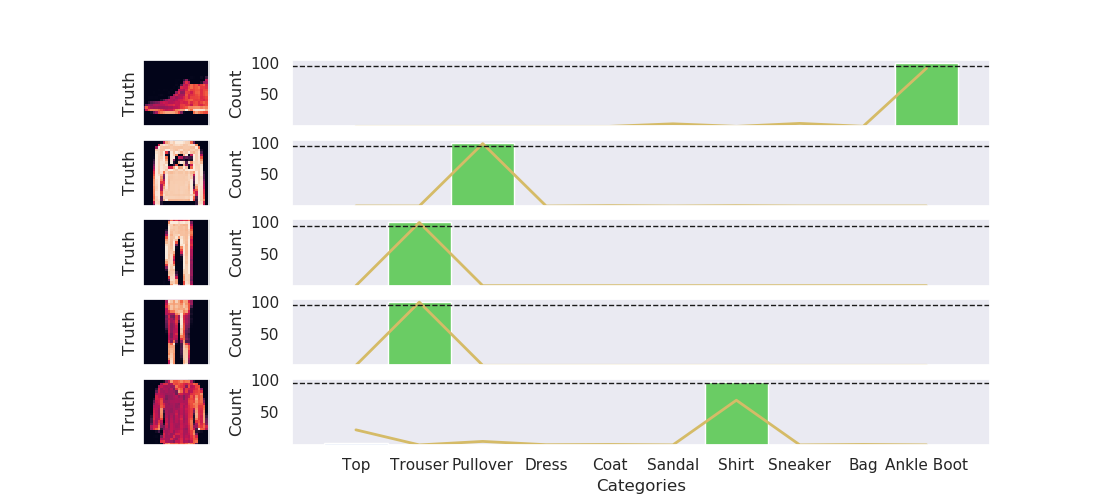}
\includegraphics[trim={3.5cm 0.7cm 3.2cm 1.3cm},clip,width=0.49\linewidth]{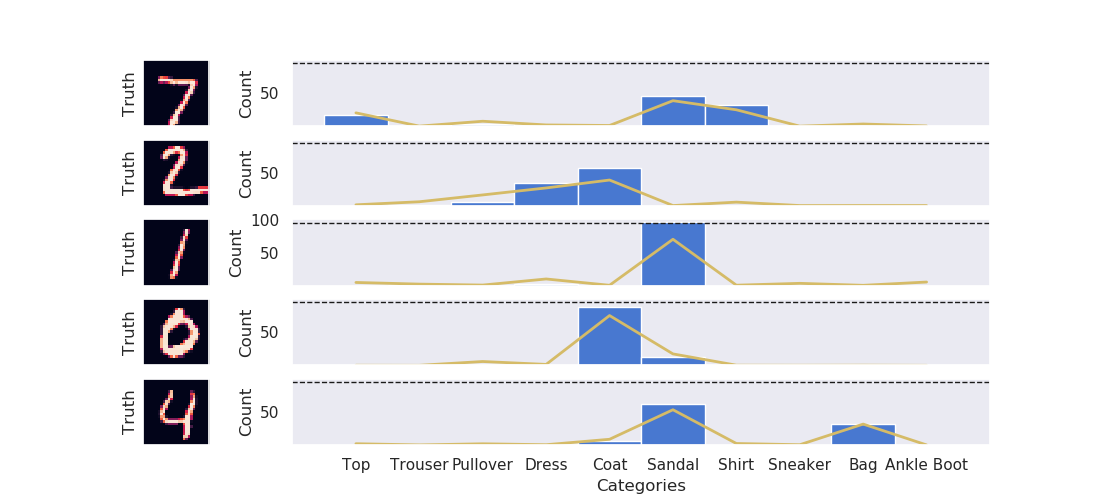}
\caption{Uncertainty related to the in-domain test data (FMNIST, left) and out-of-domain test data (MNIST, right) based on 100 samples from the posterior predictive distribution. See Figure~\ref{Fig:uncer} for additional details.}\label{Fig:funcer}
\end{figure}

Figure~\ref{Fig:entrop} shows the results on more detailed out-of-domain experiments using FMNIST data for the models trained on MNIST data and vice versa. Following~\citet{louizos2017multiplicative}, the goal now is to obtain as inconclusive results as possible (reaching ideally a uniform distribution across classes), corresponding to a large entropy. The plot shows the empirical cumulative distribution function  (CDF) of the entropies over the classified samples, where the ideal is a CDF close to the lower right corner.  
Concrete drop-out is overconfident with a distribution of test classes being far from uniform, the horseshoe prior-based approach (both before and after pruning) is the closest to uniform (but it was also closer to uniform for the in-domain predictions), whilst the 2 other baselines are in between, our approaches (both before pruning and after pruning with the median probability model) are on par with them showing that they handle out-of-domain uncertainty rather well.

\begin{figure}[!t]

\includegraphics[trim={0.2cm 0.2cm 0.3cm 0.3cm},clip, width=0.58\linewidth]{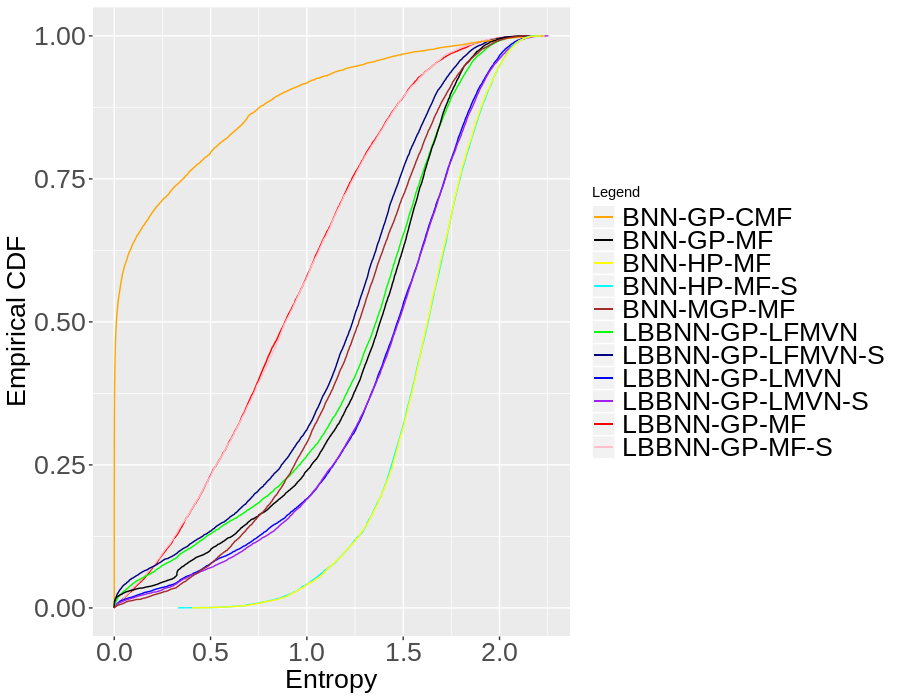}
\includegraphics[trim={0.2cm 0.2cm 8.3cm 0.3cm},clip,width=0.38\linewidth]{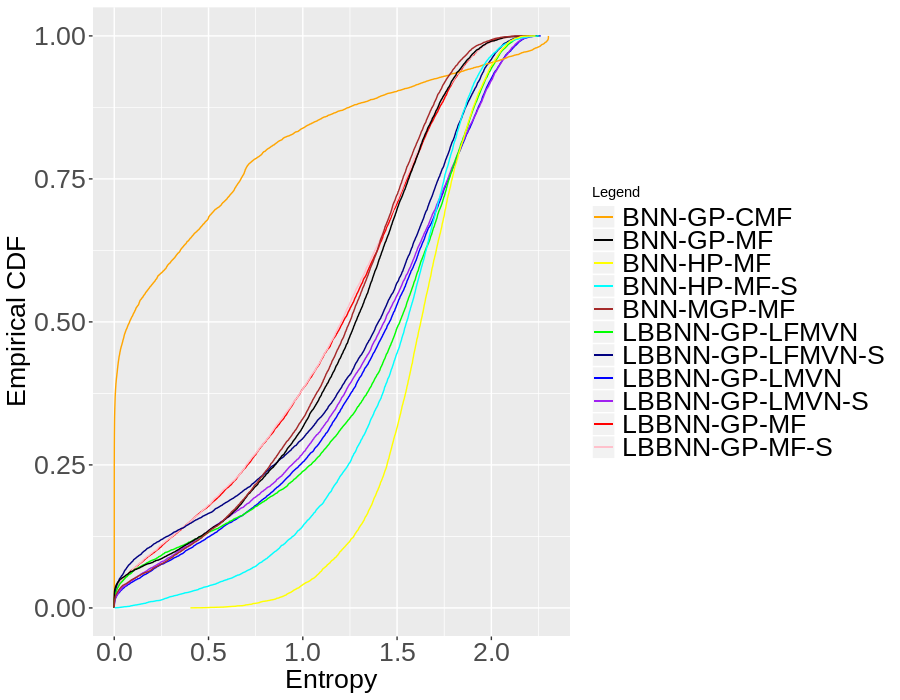}
\caption{Empirical CDF for the entropy of the marginal posterior predictive distributions trained on MNIST and applied to FMNIST (left) and vice versa (right) for simulation $s=10$. Postfix S indicates the model sparsified by an appropriate method. }\label{Fig:entrop}
\end{figure}

\paragraph{More on misclassification uncertainties}\label{ap:misclas}

Figure~\ref{Fig:missclass} shows the misclassification uncertainties associated with posterior predictive sampling. One can see that for the majority of the cases when the LBBNN-GP-MF makes a misclassification, the class certainty of the predictions is relatively low, indicating that the network is unsure. Moreover, even in these cases, the truth is typically within the 95\% credible interval of the predictions, which following \citet{posch2019variational} can be read from whether less than 95 out of 100 samples belong to a wrong class and at least 6 out of 100 samples belong to the right one. Also notice that in many of the cases of misclassification illustrated here, even a human would have serious doubts about making a decision. Here, again, very similar results were obtained for LBBNN-GP-MVN and LBBNN-GP-LFMVN approaches, but they are not reported in the paper due to space constraints.

\begin{figure}[!htbp]
\centering
\includegraphics[trim={3.5cm 3cm 3.2cm 3.5cm},clip,width=0.49\linewidth]{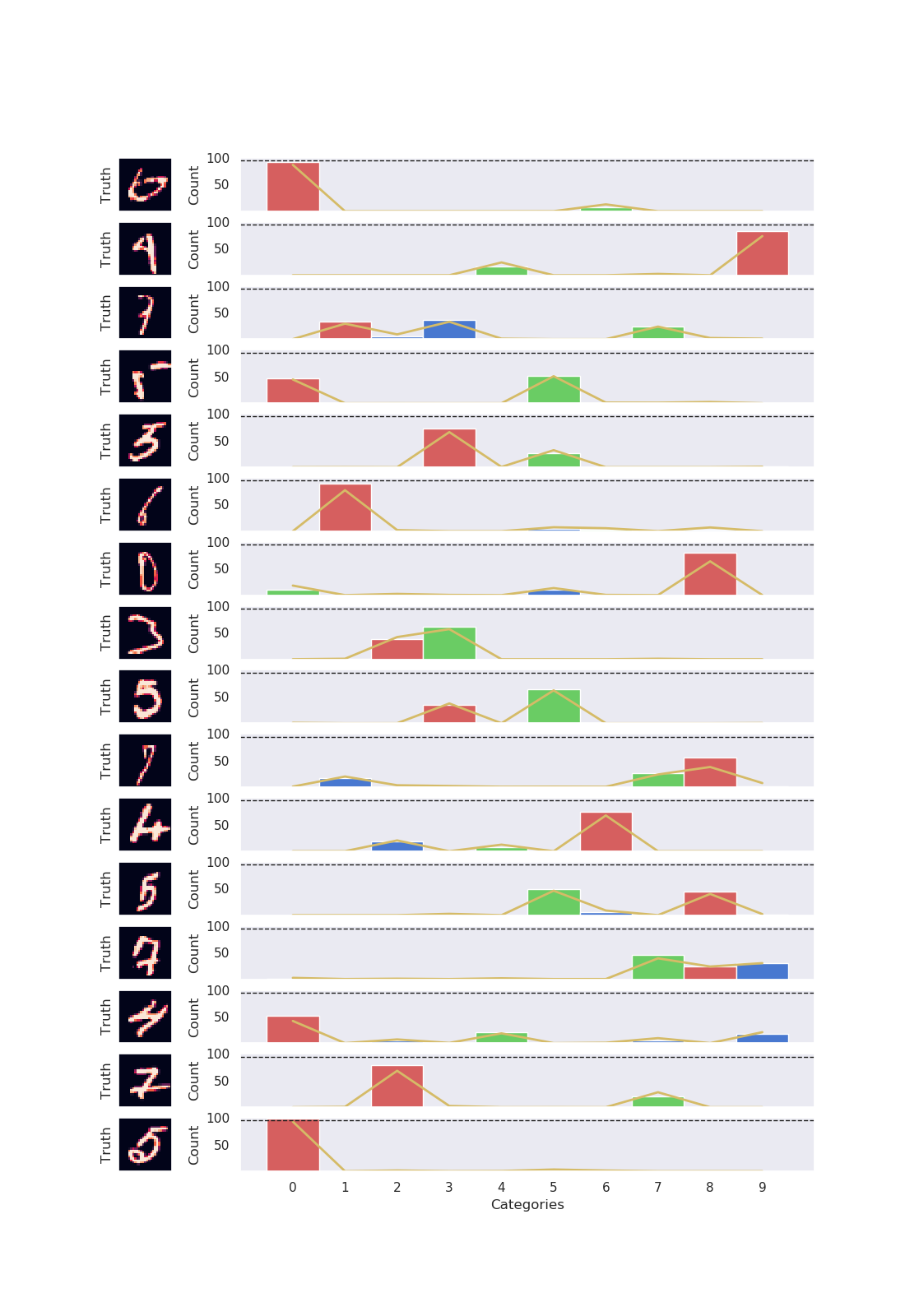}
\includegraphics[trim={3.5cm 3cm 3.2cm 3.5cm},clip,width=0.49\linewidth]{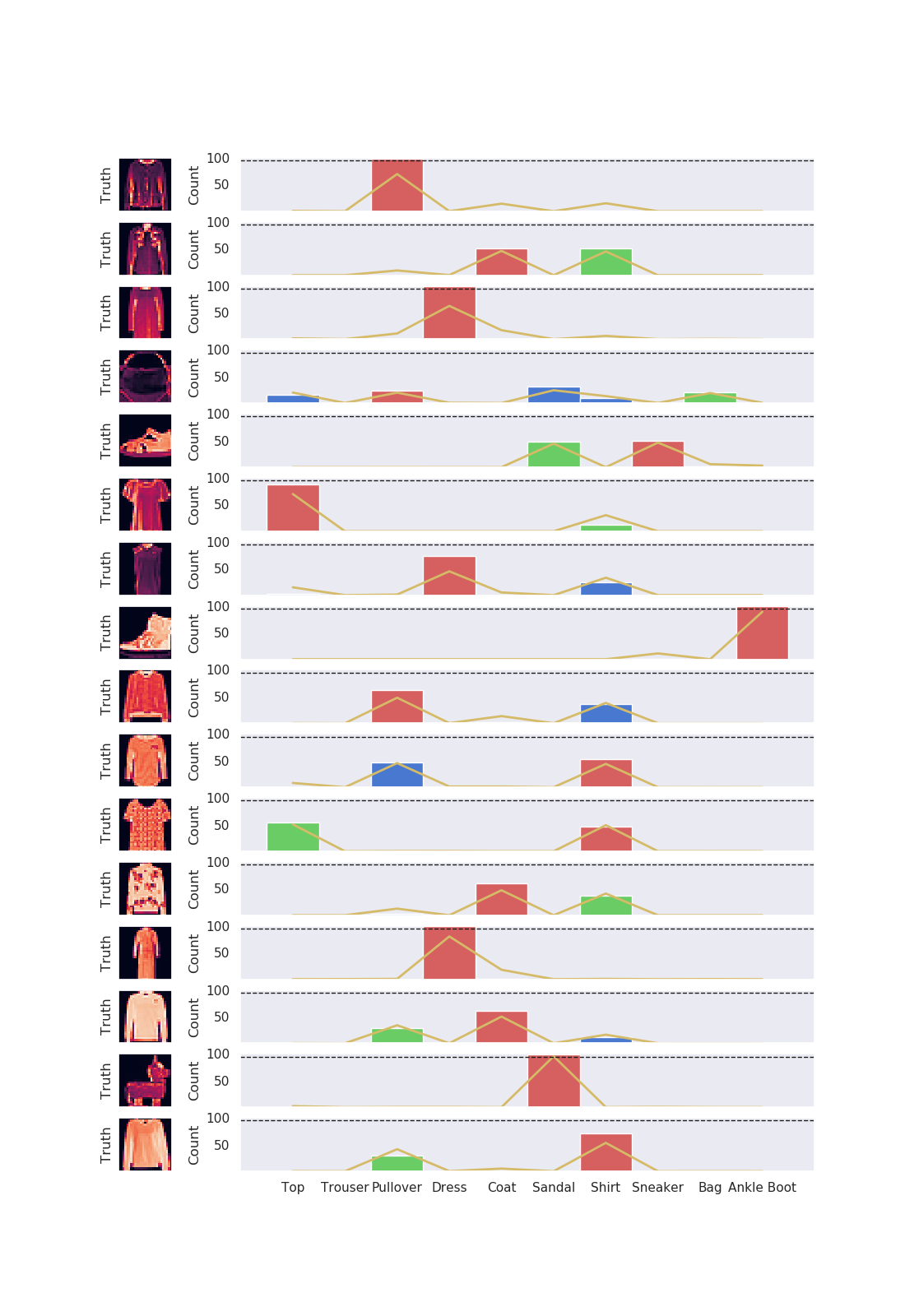}
\caption{Uncertainty based on the samples from the LBBNN-GP-MF model from the joint posterior (from simulation $s=10$) for 16 potentially wrongly  classified (under model averaging) images for MNIST data (left) and FMNIST data (right). 
 Yellow lines are model-averaged posterior class probabilities based on the Full BNN approach (in percent). 
Here, green bars are the true classes and red - the incorrectly predicted, blue bars - other samples, dashed black lines indicate the
95\% threshold for making a decision when a doubt possibility is included. The original images are depicted to the left.}\label{Fig:missclass}
\end{figure}

\section{Discussion}\label{section5}

In this paper, we have introduced the concept of Bayesian model (or structural) uncertainty in BNNs and suggested a scalable variational inference technique for approximating the joint posterior of models and the parameters of these models. Approximate posterior predictive distributions, with both models and parameters marginalized out, can be easily obtained. Furthermore, marginal inclusion probabilities give proper probabilistic interpretation to Bayesian binary dropout and allow to perform model (or architecture) selection. This comes at the price of having only one additional parameter per weight included. 

We provide image and sound classification applications of the suggested technique showing that it both allows to significantly sparsify neural networks without noticeable loss of predictive power and accurately handle the predictive uncertainty. Regarding the computational costs of optimization: For the mean-field approximations, we are introducing only one additional parameter $\alpha_{kj}^{l}$ for each weight. With underlying Gaussian structure on $\boldsymbol\alpha^{(l)}$, additional parameters of the covariance matrix are further introduced. The complexity of each optimization step is proportional to the number of parameters to optimize, thus the deterioration in terms of computational time (as demonstrated in the experiments) is not at all drastic as compared to the fully connected BNN or even FNN. For the obtained predictions the complexities for different
methods are proportional to the number of "active" parameters involved in predictions giving typically benefits to more sparse methods.  

Regarding practical recommendations, we suggest, based on our empirical results, to use LBBNN-GP-MF if one is interested in a reasonable trade-off between sparsity, predictive accuracy, and uncertainty as well as computational costs. If sparsity is not needed standard BNN-GP-MF and BNN-MGP-MF are sufficient. 

Currently, fairly simple prior distributions for both models and parameters are used. These prior distributions are assumed independent across the parameters of the neural network, which might not always be reasonable. Alternatively, both parameter and model priors can incorporate joint-dependent structures, which can further improve the sparsification of the configurations of neural networks. When it comes to the model priors with local structures and dependencies between the variables (neurons), one can mention the so-called dilution priors \citep{george2010dilution}. These priors take care of the similarities between models by down-weighting the probabilities of the models with highly correlated variables.
There are also numerous approaches to incorporate interdependencies between the model parameters via priors in different settings within simpler models \citep{smith2004bayesian, fahrmeir2001bayesian, dobra2004sparse}. Obviously, in the context of inference in the joint parameter-model settings in BNNs, more research should be done on the choice of priors. Specifically, for image analysis, it might be of interest to develop convolution-inducing priors, whilst for recurrent models, one can think of exponentially decaying parameter priors for controlling the short-long memory.

In this work, we restrict ourselves to a subclass of BNNs, defined by the inclusion-exclusion of particular weights within a given architecture. In the future, it can be of particular interest to extend the approach to the choice of the activation functions as well as the maximal depth and width of each layer of the BNN. A more detailed discussion of these possibilities and ways to proceed is given in \citet{hubin2018thesis}. 
Finally, studies of the accuracy of variational inference within these complex nonlinear models should be performed. Even within linear models, \citet{carbonetto2012scalable} have shown that the results can be strongly biased. Various approaches for reducing the bias in variational inference are developed. One can either use more flexible families of variational distributions by for example introducing auxiliary variables \citep{ranganath2016hierarchical, salimans2015markov}, normalizing flows \citep{louizos2017multiplicative}, or address Jackknife to remove the bias \citep{nowozin2018debiasing}. We leave these opportunities for further research.

The approach suggested in this paper can only implicitly perform width-selection of the architectures: Our approach allows to select the width under specific activation functions like ReLU as for a neuron with all weights switched off, i.e.  if all $\bm \gamma$'s of a given neuron are put to 0, the unit is excluded from a layer thus reducing the width. The paper does not address the depth selection of neural networks, but a similar implicit approach would be possible for the depth  as long as architectures with skip-connections are allowed (i.e. all layers are connected to the responses as well as the next layers of the networks). In such a case, it would be possible to have all nodes excluded in a specific layer of depth $k$ making only the lower depth layers influence the responses. Then, the depth uncertainty could be inferred. Addressing the depth selection, thus, is an interesting possibility for follow-up research. Also, addressing depth and width selection more explicitly through assigning targeted priors as in \citet{hubin2022flexible} could be of interest in the future. But the ideas from \citet{hubin2022flexible}  would impose more challenges for variational Bayes as compared to MCMC. At the same time, as shown in \citet{hubin2022flexible}, such a procedure is much more likely to provide highly interpretable models. Also, further research on model priors will be needed in case of the explicit width-depth selection.

Last but not least, we would like to discuss some concurrent work that appeared while our paper was in the submission/review process. Firstly, two theoretical papers on posterior consistency of Bayesian variational deep learning in general and sparse contexts were published \citep{bhattacharya2021statistical,cherief2020convergence}.  Also, \citet{bai2020efficient} justified theoretically (in an asymptotic setting) the choice of prior inclusion probabilities for the model and priors from \citet{hubin2018thesis, hubin2019combining}. They used an almost identical variational inference technique as the one proposed in \citet{hubin2019combining}. The approach from \citep{hubin2019combining} further recently found an application in genetic association studies \citep{cheng2022uncertainty}.  Finally, \citet{sun2022learning} used MCMC for inference on the model proposed in the early version of our work \citet{hubin2019combining}. These advancements show how rapidly the field develops and once again emphasize how actual and to date, the methodological developments on BNNs are in the field.

 \subsubsection*{Acknowledgments}
The authors would like to acknowledge Sean Murray (Norwegian Computing Center) for the comments on the language of the article and Dr. Pierre Lison (Norwegian Computing Center) for thoughtful discussions of the literature, potential applications, and technological tools. We also thank Dr. Petter Mostad, Department of Mathematical Sciences,
The Chalmers University of Technology and the University of Gothenburg for valuable comments on Proposition 2. We also acknowledge  constructive comments from the reviews and editorial comments we received at all stages of the publication of this article.

\clearpage
{
\bibliographystyle{ba}
\bibliography{lit}
}
\clearpage

\appendix

\section{Selected tuning parameters and results for LBBNN-GP-LFMVN}

\begin{table}[!htbp]
 \small
  \centering
  \begin{tabular}{l@{\hspace{0.4cm}}c@{\hspace{0.4cm}}ccccccc}
   \toprule
            &$A_\beta$, $A_\rho$&$A_\xi$&$A_{\omega}$&$A_\Sigma$&$A_{a_\psi}$,$A_{b_\psi}$&$A_{a_\beta}$,$A_{b_\beta}$\\
            \hline
Pre-training&0.00010&0.01000&-&0.01000&0.00100&0.00001\\
Training    &0.00010&0.00010&-&0.00010&0.00000&0.00000
\\
Post-training&0.00010&0.00000&-&0.00000&0.00000&0.00000\\
\bottomrule
  \end{tabular}
 \caption{\small Specifications of diagonal elements of $\boldsymbol A$ matrix for the step sizes of optimization routines for LBBNN-GP-LFMVN.}
 \label{tunAPP}
 
\end{table}

\begin{table}[!htbp]
 \caption{\small Performance metrics  for the MNIST data, addition to Table~3 using low factor MVN within the variational inference part.
No post-training is applied when getting these results. For further detail see Table~2 and the caption of Table~3 in the main text.}
  \label{S-t1}
   \footnotesize
  \centering
\begin{tabular}{llll|ccccc}
   \hline
\multicolumn{2}{c}{Prediction}&Model-Prior-&&\multicolumn{1}{c}{All cl}&\multicolumn{2}{c}{0.95 threshold}&Dens.&Epo.\\
$\bm\gamma$&$\bm\beta$&Method&$R$& Acc& Acc&Num.cl&level&time\\
\hline
\hline
SIM&SIM&LBBNN-GP-LFMVN&1&0.959 (0.956,0.960)&-&-&0.450&13.009\\
SIM&SIM&LBBNN-GP-LFMVN&10&0.979 (0.978,0.980)&1.000 &7760&1.000&13.009\\
ALL&EXP&LBBNN-GP-LFMVN&1&0.976 (0.976,0.978)&-&-&1.000&13.009\\ 
MED&SIM&LBBNN-GP-LFMVN&1&0.959 (0.956,0.960)&-&-&0.449&13.009\\ 
MED&SIM&LBBNN-GP-LFMVN&10&0.979 (0.978,0.980)&1.000 &7764&0.449&13.009\\ 
MED&EXP&LBBNN-GP-LFMVN&1&0.975 (0.974,0.977)&-&-&0.449&13.009\\
\bottomrule
  \end{tabular}
\end{table}

\begin{table}[!htbp]
 \caption{\small Performance metrics  for the FMNIST data addition to Table~\ref{t11} using low factor MVN within the variational inference part.
No post-training is applied when getting these results. For further detail see Table~2 and the caption of Table~3 in the main text.}
  \label{S-t2}
   \footnotesize
  \centering
\begin{tabular}{llll|ccccc}
   \hline
\multicolumn{2}{c}{Prediction}&Model-Prior-&&\multicolumn{1}{c}{All cl}&\multicolumn{2}{c}{0.95 threshold}&Dens.&Epo.\\
$\bm\gamma$&$\bm\beta$&Method&$R$& Acc& Acc&Num.cl&level&time\\
\hline
SIM&SIM&LBBNN-GP-LFMVN&1&0.849 (0.845,0.853)&-&-&0.456&12.949\\
SIM&SIM&LBBNN-GP-LFMVN&10&0.876 (0.874,0.880)&0.996 &4485&1.000&12.949\\
ALL&EXP&LBBNN-GP-LFMVN&1&0.864 (0.862,0.867)&-&-&1.000&12.949\\ 
MED&SIM&LBBNN-GP-LFMVN&1&0.849 (0.844,0.852)&-&-&0.455&12.949\\ 
MED&SIM&LBBNN-GP-LFMVN&10&0.877 (0.875,0.878)&0.996 &4486&0.455&12.949\\ 
MED&EXP&LBBNN-GP-LFMVN&1&0.862 (0.858,0.864)&-&-&0.455&12.949\\
\bottomrule
  \end{tabular}
\end{table}

\begin{table}[!t]
 \caption{\small Performance metrics  for the PHONEME data addition to Table~\ref{t111} using  low factor MVN within the variational inference part.
No post-training is applied when getting these results. For further detail see Table~2 and the caption of Table~3 in the main text.}
  \label{S-t3}
  \footnotesize
  \centering
\begin{tabular}{llll|ccccc}
   \hline
\multicolumn{2}{c}{Prediction}&Model-Prior-&&\multicolumn{1}{c}{All cl}&\multicolumn{2}{c}{0.95 threshold}&Dens.&Epo.\\
$\bm\gamma$&$\bm\beta$&Method&$R$& Acc& Acc&Num.cl&level&time\\
\hline
SIM&SIM&LBBNN-GP-LFMVN&1&0.918 (0.906,0.928)&-&-&0.497&0.472\\
SIM&SIM&LBBNN-GP-LFMVN&10&0.929 (0.926,0.934)
&0.994 &663&1.000&0.472\\
ALL&EXP&LBBNN-GP-LFMVN&1&0.921 (0.915,0.931)&-&-&1.000&0.472\\ 
MED&SIM&LBBNN-GP-LFMVN&1&0.918 (0.909,0.930)&-&-&0.497&0.472\\ 
MED&SIM&LBBNN-GP-LFMVN&10&0.929 (0.925,0.933)&0.995 &664&0.497&0.472\\ 
MED&EXP&LBBNN-GP-LFMVN&1& 0.917 (0.912,0.925)&-&-&0.497&0.472\\
\bottomrule
  \end{tabular}
\end{table}

\clearpage
\newpage

\section{Results for frequentist neural network with various degrees of pruning}

In this experiments, {we included the results of standard magnitude-based pruning of a frequentist neural network of the same configuration as we used for the BNNs in the main paper for MNIST, FMNIST, and PHONEMNE data sets}. There, in Tables S-5, S-6, and S-7 we show that under all sparsity levels (corresponding to those obtained with the Bayesian approaches) for all three datasets addressed, all Bayesian approaches outperform the frequentist counterpart in terms of predictive accuracy. Yet, the fully connected dense FNN performs on par with the Bayesian versions. Here, sparsity was obtained by removing a corresponding share of weights/neurons with the lowest magnitude. 

\begin{table}[H]
 \caption{\small Performance metrics of frequentist neural network under various degrees of pruning for the MNIST data.}
  \label{S-t4}
  \footnotesize
  \centering
\begin{tabular}{l|l|l|c}
   \hline
Dens. level&{Acc. neuron pruning}&{Acc. weight pruning}&Epo.time\\
\hline
1.000&98.110 (98.040,98.350)&98.110 (98.040,98.350)&1.360\\
0.500&98.030 (97.720,98.080)&87.170 (78.260,90.730)&1.360\\
0.226&96.600 (95.570,97.150)&36.915 (32.600,46.910)&1.360\\
0.194&96.130 (94.680,96.670)&32.320 (27.320,38.150)&1.360\\
0.180&95.710 (93.290,96.730)&31.685 (25.010,35.840)&1.360\\
0.163&95.070 (91.040,96.170)&29.760 (23.810,33.340)&1.360\\
0.090&80.050 (75.060,90.120)&18.235 (14.500,21.520)&1.360\\
0.079&77.730 (68.910,86.090)&19.055 (15.270,25.990)&1.360\\
\bottomrule
  \end{tabular}
\end{table}

\begin{table}[H]
 \caption{\small Performance metrics of frequentist neural network under various degrees of pruning for the FMNIST data.}
  \label{S-t5}
  \footnotesize
  \centering
\begin{tabular}{l|l|l|c}
   \hline
Dens. level&{Acc. neuron pruning}&{Acc. weight pruning}&Epo.time\\
\hline
1.000&89.470 (86.240,89.790)&89.470 (86.240,89.790)&1.260\\
0.500&85.150 (80.040,87.180)&29.865 (21.370,41.080)&1.260\\
0.302&67.610 (53.110,75.760)&17.175 (14.020,20.660)&1.260\\
0.156&41.595 (37.520,45.510)&11.460 (9.370,20.720)&1.260\\
0.129&39.040 (30.090,44.980)&10.515 (9.300,18.490)&1.260\\
0.120&39.775 (26.140,44.480)&09.980 (8.130,18.610)&1.260\\
0.108&38.750 (23.890,43.660)&10.325 (8.080,20.110)&1.260\\
0.094&36.340 (23.640,41.050)&09.790 (8.040,25.000)&1.260\\
\bottomrule
  \end{tabular}
\end{table}

\begin{table}[H]
 \caption{\small Performance metrics of frequentist neural network under various degrees of pruning for the PHONEMNE data.}
  \label{S-t6}
  \footnotesize
  \centering
\begin{tabular}{l|l|l|c}
   \hline
Dens. level&{Acc. neuron pruning}&{Acc. weight pruning}&Epo.time\\
\hline
1.000&92.400 (92.1,92.9)&92.500 (92.200,92.700)&0.029\\
0.600&92.250 (91.8,93.2)&88.050 (86.100,90.500)&0.029\\
0.509&92.100 (91.5,92.7)&82.650 (78.100,85.400)&0.029\\
0.457&92.100 (91.6,92.7)&80.950 (75.700,86.500)&0.029\\
0.371&91.850 (90.4,92.6)&76.050 (68.300,82.200)&0.029\\
0.307&91.700 (90.7,92.3)&71.950 (66.500,80.200)&0.029\\
0.255&91.000 (90.4,92.8)&67.900 (58.600,80.200)&0.029\\
0.225&90.700 (90.2,92.6)&64.950 (51.800,77.300)&0.029\\
\bottomrule
  \end{tabular}
\end{table}

\clearpage

\section{Results based on post-training}

\begin{table}[!htbp]
 \caption{\small Performance metrics  for the MNIST data for suggested in the article Bayesian approaches to BNN. The results after post-training are reported here. For further detail see Table~2 and the caption of Table~3. }
  \label{t1p}
  \small
  \centering
\begin{tabular}{llll|cccc}
   \hline
   &&&&\multicolumn{4}{c}{MNIST}\\
\multicolumn{4}{c|}{}&\multicolumn{1}{c}{All cl}&\multicolumn{2}{c}{0.95 threshold}&Density\\
$\bm\gamma$&$\bm\beta$&Method&$R$& Acc& Acc&Num.cl&level\\
    \hline
SIM&SIM&LBBNN-GP-MF&1&0.967 (0.966,0.969)&-&-&0.090\\
SIM&SIM&LBBNN-GP-MF&10&0.980 (0.979,0.982)&0.999 &8346&1.000\\
ALL&EXP&LBBNN-GP-MF&1&0.982 (0.980,0.983)&-&-&1.000\\ 
MED&SIM&LBBNN-GP-MF&1&0.969 (0.966,0.972)&-&-&0.079\\
MED&SIM&LBBNN-GP-MF&10&0.980 (0.979,0.982)&0.999 &8472&0.079\\
MED&EXP&LBBNN-GP-MF&1&0.981 (0.980,0.984)&-&-&0.079\\
\hline
SIM&SIM&LBBNN-GP-MVN&1&0.967 (0.965,0.970)&-&-&0.180\\
SIM&SIM&LBBNN-GP-MVN&10&0.979 (0.977,0.981)&1.000 &7994&1.000\\
ALL&EXP&LBBNN-GP-MVN&1&0.980 (0.979,0.981)&-&-&1.000\\ 
MED&SIM&LBBNN-GP-MVN&1&0.973 (0.971,0.977)&-&-&0.163\\
MED&SIM&LBBNN-GP-MVN&10&0.978 (0.977,0.979)&1.000 &8107&0.163\\
MED&EXP&LBBNN-GP-MVN&1&0.976 (0.974,0.977)&-&-&0.163\\
\hline
SIM&SIM&LBBNN-GP-LFMVN&1&0.971 (0.969,0.973)&-&-&0.450\\
SIM&SIM&LBBNN-GP-LFMVN&10&0.978 (0.976,0.980)&0.999 &8366&1.000\\
ALL&EXP&LBBNN-GP-LFMVN&1&0.979 (0.978,0.980)&-&-&1.000\\ 
MED&SIM&LBBNN-GP-LFMVN&1&0.973 (0.971,0.977)&-&-&0.449\\
MED&SIM&LBBNN-GP-LFMVN&10&0.978 (0.977,0.979)&0.999 &8645&0.449\\
MED&EXP&LBBNN-GP-LFMVN&1&0.978 (0.976,0.979)&-&-&0.449\\
\hline
SIM&SIM&BNN-GP-CMF&1&0.982 (0.894,0.984)&-&-&0.226\\
SIM&SIM&BNN-GP-CMF&10&0.984 (0.896,0.986)&0.995 &9586&1.000\\
ALL&EXP&BNN-GP-CMF&1&0.983 (0.894,0.984)&-&-&1.000\\
\hline
PRN&SIM&BNN-HP-MF&1&0.967 (0.965,0.968)&-&-&0.194\\
PRN&SIM&BNN-HP-MF&10&0.982 (0.981,0.983)&1.000 &0007&0.194\\
PRN&EXP&BNN-HP-MF&1&0.966 (0.964,0.969)&-&-&0.194\\
\bottomrule
  \end{tabular}
\end{table}

\begin{table}[!htbp]
 \caption{\small Performance metrics  for the FMNIST data for the suggested in the article Bayesian approaches to BNN. The results after post-training are reported here. For further detail  see Table~2 and the caption of Table~3.}
  \label{t11p}
  \small
  \centering
\begin{tabular}{llll|cccc}
   \hline
   &&&&\multicolumn{4}{c}{FMNIST}\\
\multicolumn{4}{c|}{}&\multicolumn{1}{c}{All cl}&\multicolumn{2}{c}{0.95 threshold}&Density\\
$\bm\gamma$&$\bm\beta$&Method&$R$& Acc& Acc&Num.cl&level\\
    \hline
SIM&SIM&LBBNN-GP-MF&1&0.863 (0.862,0.869)&-&-&0.120\\
SIM&SIM&LBBNN-GP-MF&10&0.884 (0.881,0.886)&0.995 &4932&1.000\\
ALL&EXP&LBBNN-GP-MF&1&0.882 (0.879,0.887)&-&-&1.000\\
MED&SIM&LBBNN-GP-MF&1&0.866 (0.865,0.869)&-&-&0.108\\
MED&SIM&LBBNN-GP-MF&10&0.885 (0.882,0.887)&0.995 &4994&0.108\\
MED&EXP&LBBNN-GP-MF&1&0.882 (0.878,0.886)&-&-&0.108\\
\hline
SIM&SIM&LBBNN-GP-MVN&1&0.859 (0.857,0.862)&-&-&0.156\\
SIM&SIM&LBBNN-GP-MVN&10&0.880 (0.874,0.881)&0.996 &4615&1.000\\
ALL&EXP&LBBNN-GP-MVN&1&0.876 (0.872,0.877)&-&-&1.000\\ 
MED&SIM&LBBNN-GP-MVN&1&0.863 (0.860,0.865)&-&-&0.129\\
MED&SIM&LBBNN-GP-MVN&10&0.878 (0.874,0.880)&0.995 &4801&0.129\\
MED&EXP&LBBNN-GP-MVN&1&0.873 (0.870,0.875)&-&-&0.129\\
\hline
SIM&SIM&LBBNN-GP-LFMVN&1&0.847 (0.844,0.850)&-&-&0.456\\
SIM&SIM&LBBNN-GP-LFMVN&10&0.875 (0.873,0.878)&0.996 &4431&1.000\\
ALL&EXP&LBBNN-GP-LFMVN&1&0.866 (0.859,0.868)&-&-&1.000\\ 
MED&SIM&LBBNN-GP-LFMVN&1&0.846 (0.844,0.849)&-&-&0.455\\
MED&SIM&LBBNN-GP-LFMVN&10&0.876 (0.873,0.879)&0.996 &4426&0.455\\
MED&EXP&LBBNN-GP-LFMVN&1&0.864 (0.860,0.865)&-&-&0.455\\
    \hline
SIM&SIM&BNN-GP-CMF&1&0.897 (0.820,0.899)&-&-&0.094\\
SIM&SIM&BNN-GP-CMF&10&0.897 (0.823,0.902)&0.943 &8826&1.000\\
ALL&EXP&BNN-GP-CMF&1&0.896 (0.820,0.901)&-&-&1.000\\
\hline
PRN&SIM&BNN-HP-MF&1&0.867 (0.864,0.871)&-&-&0.302\\
PRN&SIM&BNN-HP-MF&10&0.888 (0.887,0.890)&1.000 &0147&0.302\\
PRN&EXP&BNN-HP-MF&1&0.868 (0.864,0.869)&-&-&0.302\\
\bottomrule
  \end{tabular}
\end{table}

\begin{table}[!htbp]
 \caption{\small Performance metrics  for the PHONEME data for the suggested in the article Bayesian approaches to BNN. The results after post-training are reported here. For further detail see Table~2 and the caption of Table~3.}
  \label{t111p}
  \small
  \centering
\begin{tabular}{llll|cccc}
   \hline
   &&&&\multicolumn{4}{c}{FMNIST}\\
\multicolumn{4}{c|}{}&\multicolumn{1}{c}{All cl}&\multicolumn{2}{c}{0.95 threshold}&Density\\
$\bm\gamma$&$\bm\beta$&Method&$R$& Acc& Acc&Num.cl&level\\
    \hline
SIM&SIM&LBBNN-GP-MF&1&0.917 (0.895,0.928)&-&-&0.120\\
SIM&SIM&LBBNN-GP-MF&10&0.927 (0.921,0.931)&0.991 &703&1.000\\
ALL&EXP&LBBNN-GP-MF&1&0.924 (0.923,0.929)
&-&-&1.000\\
MED&SIM&LBBNN-GP-MF&1&0.924 (0.910,0.930)&-&-&0.108\\
MED&SIM&LBBNN-GP-MF&10&0.924 (0.911,0.932)&0.981 &772&0.108\\
MED&EXP&LBBNN-GP-MF&1&0.925 (0.909,0.931)&-&-&0.108\\
\hline
SIM&SIM&LBBNN-GP-MVN&1&0.922 (0.914,0.928)&-&-&0.255\\
SIM&SIM&LBBNN-GP-MVN&10&0.931 (0.926,0.935)&0.993 &682&1.000\\
ALL&EXP&LBBNN-GP-MVN&1&0.927 (0.917,0.932)&-&-&1.000\\ 
MED&SIM&LBBNN-GP-MVN&1&0.923 (0.919,0.931)&-&-&0.225\\ 
MED&SIM&LBBNN-GP-MVN&10&0.928 (0.924,0.933)&0.991 &691&0.225\\ 
MED&EXP&LBBNN-GP-MVN&1&0.924 (0.915,0.930)&-&-&0.225\\
\hline
SIM&SIM&LBBNN-GP-LFMVN&1&0.919 (0.910,0.926)&-&-&0.497\\
SIM&SIM&LBBNN-GP-LFMVN&10&0.929 (0.925,0.931)
&0.994 &678&1.000\\
ALL&EXP&LBBNN-GP-LFMVN&1&0.921 (0.912,0.933)&-&-&1.000\\ 
MED&SIM&LBBNN-GP-LFMVN&1&0.917 (0.895,0.926)&-&-&0.497\\ 
MED&SIM&LBBNN-GP-LFMVN&10&0.928 (0.923,0.932)&0.994 &676&0.497\\ 
MED&EXP&LBBNN-GP-LFMVN&1&0.917 (0.909,0.926)&-&-&0.497\\
    \hline
SIM&SIM&BNN-GP-CMF&1&0.882 (0.716,0.904)&-&-&0.509\\
SIM&SIM&BNN-GP-CMF&10&0.921 (0.918,0.930)&0.960 &189&1.000\\
ALL&EXP&BNN-GP-CMF&1&0.877 (0.697,0.909)
&-&-&1.000\\
\hline
PRN&SIM&BNN-HP-MF&1&0.913 (0.909,0.926)&-&-&0.457\\
PRN&SIM&BNN-HP-MF&10&0.917 (0.916,0.922)&0.936 &037&0.457\\
PRN&EXP&BNN-HP-MF&1&0.917 (0.914,0.924)&-&-&0.457\\
\bottomrule
  \end{tabular}
\end{table}

\begin{figure}[!htbp]

\includegraphics[trim={0.2cm 0.2cm 0.3cm 0.3cm},clip, width=0.58\linewidth]{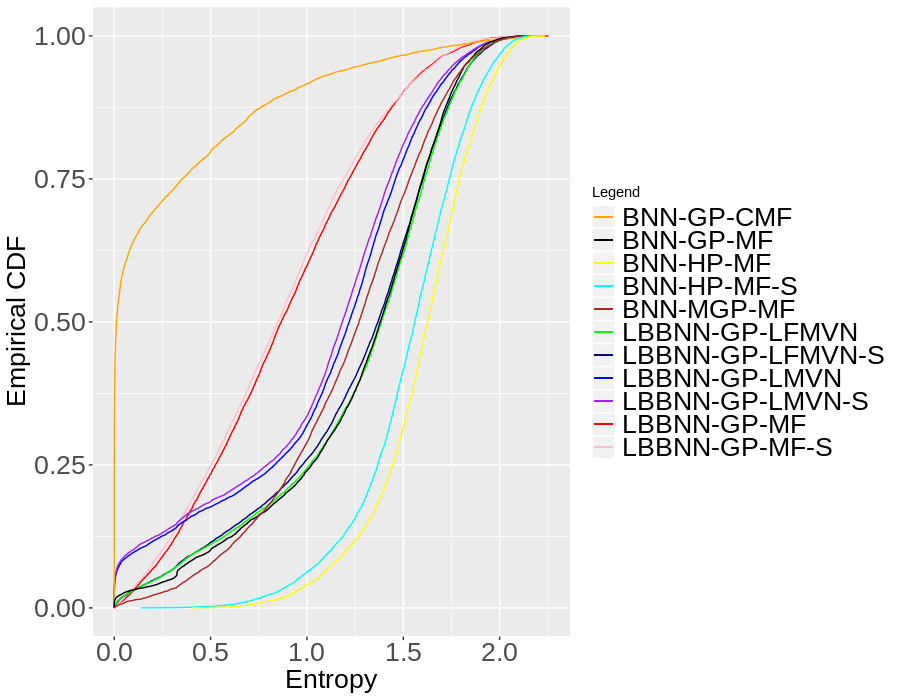}
\includegraphics[trim={0.2cm 0.2cm 8.3cm 0.3cm},clip,width=0.38\linewidth]{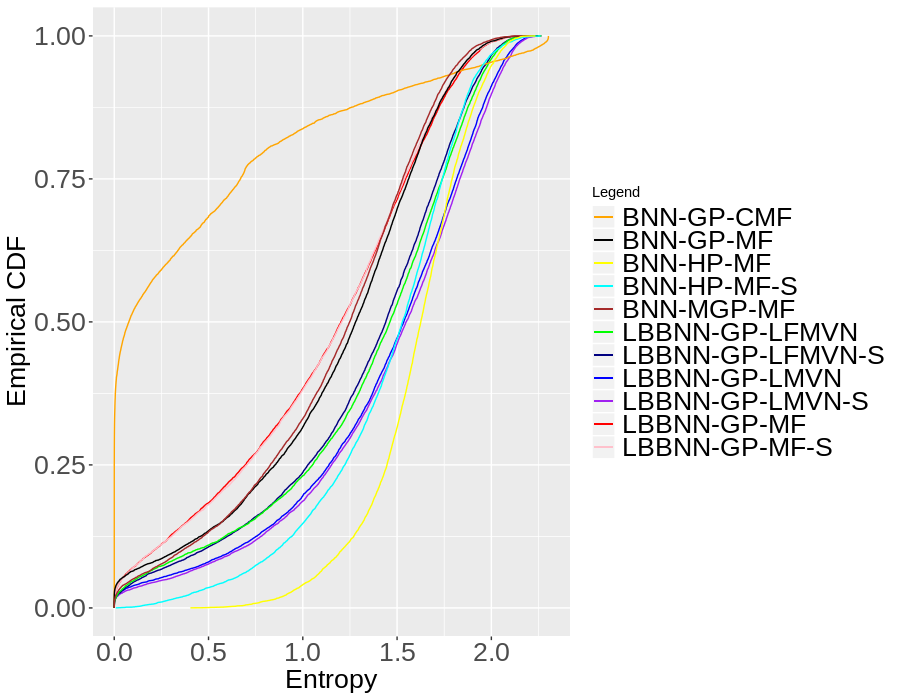}
\caption{Empirical CDF for the entropy of the marginal posterior predictive distributions trained on MNIST and applied to FMNIST (left) and vice versa (right) for simulation $s=10$ after post-training. Postfix S indicates the model sparsified by an appropriate method. }\label{Fig:entroppt}
\end{figure}

\clearpage
\newpage

\section{Results on LBBNN with fixed parameters of the prior}

In this section, we add results for MNIST, and FMNIST datasets obtained with the basic LBBNN with Gaussian priors on the weights and Bernoulli priors on the inclusion of the weights, where no further hyperpriors are assumed. Prior inclusion probabilities here correspond to the BIC penalties, i.e. $\psi^{(l)} = \exp(-2\log n)$ and prior variances of the slav components are set to 1. Also, basic mean field variational distributions are used here.

\begin{table}[H]
 \caption{\small Performance metrics  for the MNIST  data for suggested in the article Bayesian approaches to BNN. The results without post-training are reported here. For further detail see Table~2 and the caption of Table~3. }
  \label{t1o}
  \small
  \centering
\begin{tabular}{llll|cccc}
   \hline
   &&&&\multicolumn{4}{c}{MNIST}\\
\multicolumn{4}{c|}{}&\multicolumn{1}{c}{All cl}&\multicolumn{2}{c}{0.95 threshold}&Density\\
$\bm\gamma$&$\bm\beta$&Method&$R$& Acc& Acc&Num.cl&level\\
    \hline
SIM&SIM&LBBNN-GP-MF&1&0.958 (0.954,0.960)&-&-&0.056\\
SIM&SIM&LBBNN-GP-MF&10&0.967 (0.966,0.971)&0.999 &7064&0.084\\
ALL&EXP&LBBNN-GP-MF&1&0.969 (0.967,0.970)&-&-&1.000\\
MED&SIM&LBBNN-GP-MF&1&0.961 (0.957,0.964)&-&-&0.051\\ 
MED&SIM&LBBNN-GP-MF&10&0.964 (0.962,0.967)&0.998 &7441&0.051\\
MED&EXP&LBBNN-GP-MF&1&0.965 (0.963,0.968)&-&-&0.051\\
\bottomrule
  \end{tabular}
\end{table}

\begin{table}[H]
 \caption{\small Performance metrics  for the FMNIST data  for suggested in the article Bayesian approaches to BNN. The results without post-training are reported here. For further detail see Table~2 and the caption of Table~3. }
  \label{t1o}
  \small
  \centering
\begin{tabular}{llll|cccc}
   \hline
   &&&&\multicolumn{4}{c}{FMNIST}\\
\multicolumn{4}{c|}{}&\multicolumn{1}{c}{All cl}&\multicolumn{2}{c}{0.95 threshold}&Density\\
$\bm\gamma$&$\bm\beta$&Method&$R$& Acc& Acc&Num.cl&level\\
    \hline
SIM&SIM&LBBNN-GP-MF&1&0.854 (0.850,0.858)&-&-&0.066\\
SIM&SIM&LBBNN-GP-MF&10&0.867 (0.863,0.870)&0.996 &4097&0.083\\
ALL&EXP&LBBNN-GP-MF&1&0.866 (0.864,0.874)&-&-&1.000\\
MED&SIM&LBBNN-GP-MF&1&0.858 (0.854,0.865)&-&-&0.065\\ 
MED&SIM&LBBNN-GP-MF&10&0.863 (0.859,0.869)&0.993 &4347&0.065\\
MED&EXP&LBBNN-GP-MF&1&0.863 (0.859,0.870)&-&-&0.065\\
\bottomrule
  \end{tabular}
\end{table}

\begin{table}[H]
 \caption{\small Medians and standard deviations of the average (per layer) marginal inclusion probability (see the text for the definition) for our model for both MNIST and FMNIST data across 10 repeated experiments.}
 \label{t3o}
  \small
  \centering
  \begin{tabular}{l@{\hspace{0.4cm}}c@{\hspace{0.4cm}}ccc}
   \toprule
&\multicolumn{2}{l}{\textbf{MNIST data}}&\multicolumn{2}{l}{\textbf{FMNIST data}}\\
\hline
Layer& Med.  & SD.& Med.  & SD.\\\hline
$\rho(\gamma^{(1)}|\D)$&0.0520&0.0005&0.0665&0.0004\\
$\rho(\gamma^{(2)}|\D)$&0.0598&0.0003&0.0613&0.0005\\
$\rho(\gamma^{(3)}|\D)$&0.2217&0.0064&0.2013&0.0051\\
\hline
\bottomrule
  \end{tabular}

\end{table}

\begin{figure}[h!]
\centering
\includegraphics[width=0.6\linewidth]{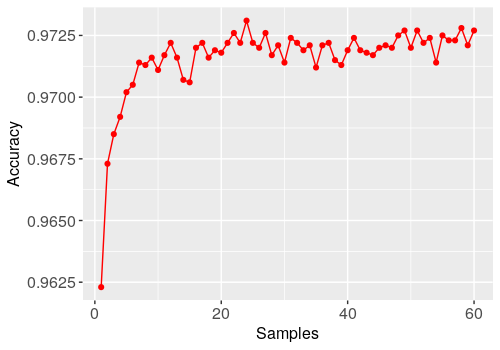}

\caption{\small  Accuracy of predictions versus the number of samples from the joint posterior of models and parameters for simulation $i=10$ on the MNIST data.}\label{Fig:predsamples}
\end{figure}

\begin{figure}[h!]
\centering

\includegraphics[width=0.6\linewidth]{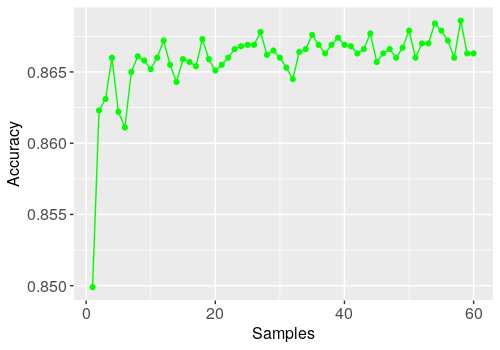}

\caption{\small  Accuracy of predictions versus the number of samples from the joint posterior for the FMNIST data.}\label{Fig:samples}
\end{figure}

\clearpage
\newpage

\section{Results with varying widths of the layers}

\begin{figure}[h!]
\centering

\includegraphics[width=1\linewidth]{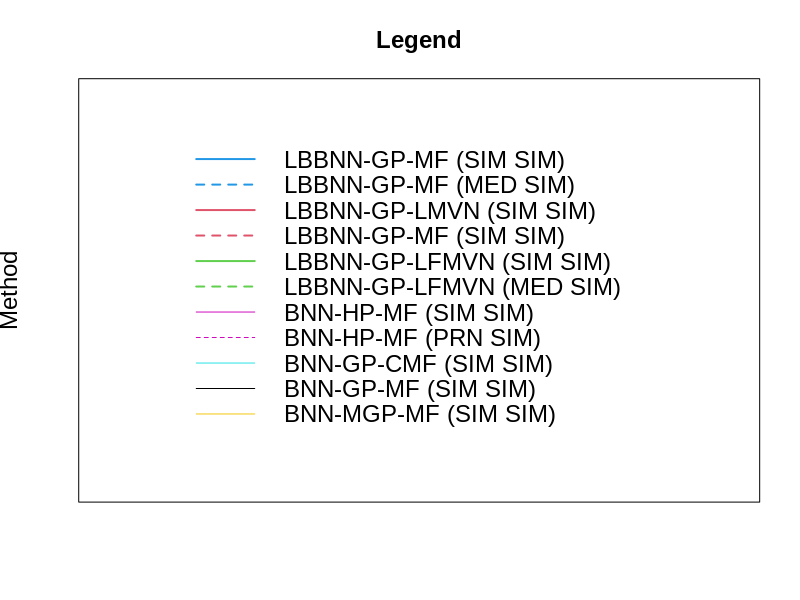}

\caption{\small  Legend for lines used in Figures \ref{Fig:widthsim}-\ref{Fig:widthdens} for experiments width varying width factor for the PHONEMNE data.}\label{Fig:widthsim}
\end{figure}

\begin{figure}[h!]
\centering
\includegraphics[width=1\linewidth]{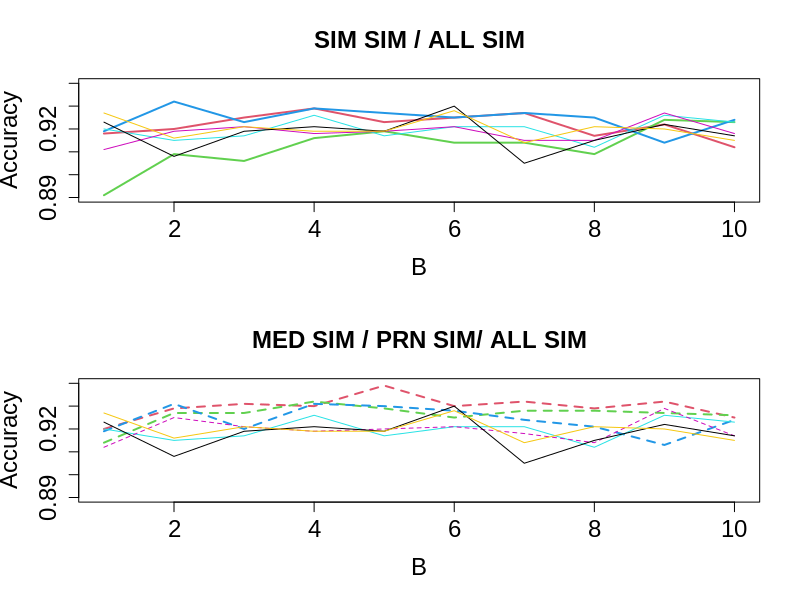}

\caption{\small  Accuracy of predictions versus the width factor for the PHONEMNE data: Top graph (SIM, SIM) for LBBNN-GP-MF, LBBNN-GP-MVN, LBBNN-GP-LFMVN, BNN-GP-CMF, BNN-HP-MF; (ALL, SIM) for BNN-GP-MF, BNN-MGP-MF. Bottom graph (MED, SIM) for LBBNN-GP-MF, LBBNN-GP-MVN, LBBNN-GP-LFMVN; (SIM SIM) for  BNN-GP-CMF; (PRN SIM) for BNN-HP-MF; (ALL, SIM) for BNN-GP-MF, BNN-MGP-MF.}\label{Fig:widthsim}
\end{figure}

\begin{figure}[h!]
\centering

\includegraphics[width=1\linewidth]{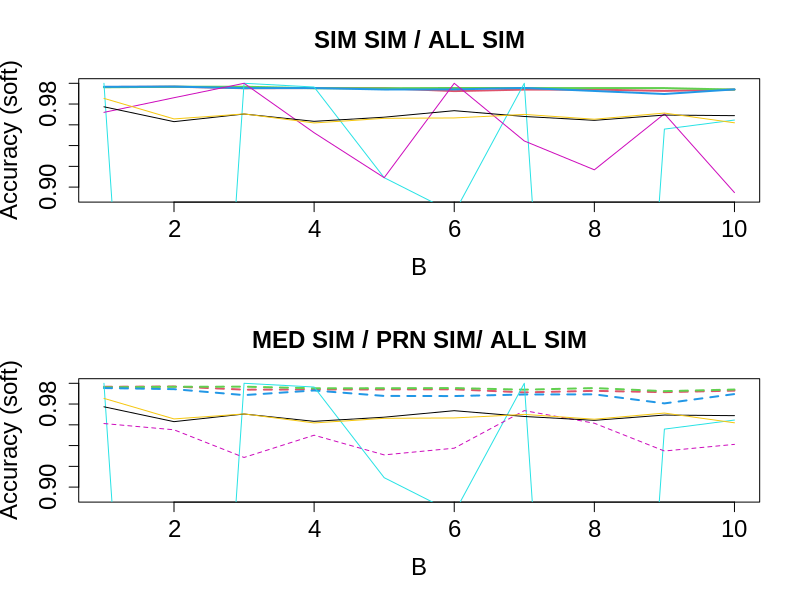}

\caption{\small Accuracy of uncertainty aware predictions (0.95 posterior model averaged probability threshold) versus the width factor for the PHONEMNE data: Top graph (SIM, SIM) for LBBNN-GP-MF, LBBNN-GP-MVN, LBBNN-GP-LFMVN, BNN-GP-CMF, BNN-HP-MF; (ALL, SIM) for BNN-GP-MF, BNN-MGP-MF. Bottom graph (MED, SIM) for LBBNN-GP-MF, LBBNN-GP-MVN, LBBNN-GP-LFMVN; (SIM SIM) for  BNN-GP-CMF; (PRN SIM) for BNN-HP-MF; (ALL, SIM) for BNN-GP-MF, BNN-MGP-MF.}\label{Fig:widthsoft}
\end{figure}

\begin{figure}[h!]
\centering

\includegraphics[width=1\linewidth]{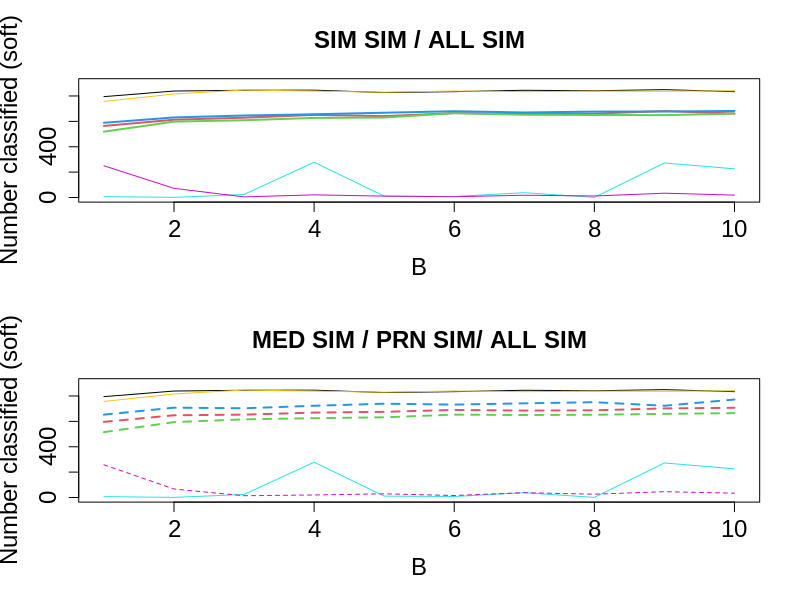}

\caption{\small  Number of uncertainty aware predictions (0.95 posterior model averaged probability threshold) versus the width factor for the PHONEMNE data: further details in the caption to Figure \ref{Fig:widthsoft}.}\label{Fig:widthdens}
\end{figure}

\begin{figure}[h!]
\centering

\includegraphics[width=0.7\linewidth]{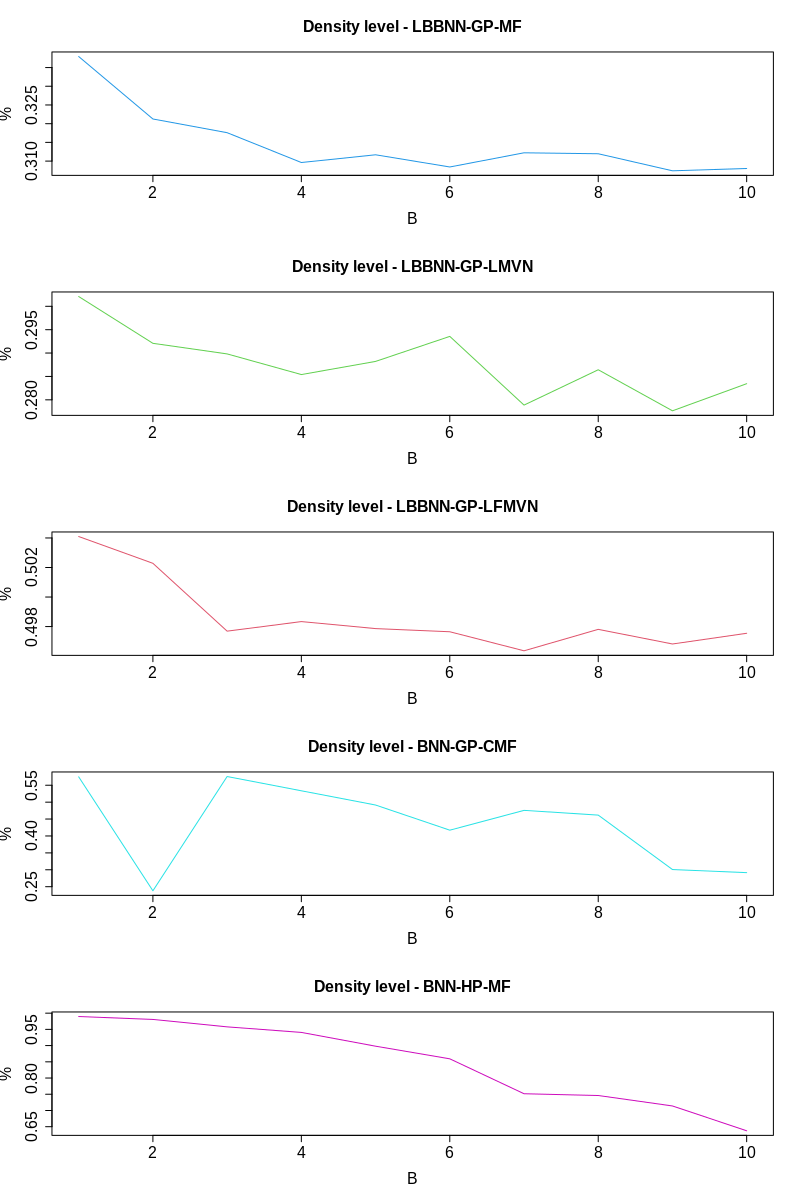}

\caption{\small  Densities for various Bayesian sparsifying methods for different widths of the layers on the PHONEMNE data.}\label{Fig:samples}
\end{figure}

In this experiment, we used the settings used for the PHONEMNE data study from Section 4 of the main paper, but the width of all layers was varied as 40B-60B-60B, and B was changed from 1 to 10. The results are summarized in Figures \ref{Fig:widthsim}-\ref{Fig:widthdens}. As expected all the sparsity-inducing methods results in an increase in sparsity levels as the width increases. Predictions increase in general from B = 1 to B = 4 and stabilize at B = 4. Last but not least, for every fixed B, we have the same general concussions as those reported in the PHONEMME data set in the main paper: LBBNNs are giving a good and stable trade-off between predictive accuracy, uncertainty aware accuracy, and sparsity for all widths' configurations of the models.

\newpage
\clearpage

\section{Results on correlation structures between posterior inclusion of weights for PHONEMNE data}

In this section, we report correlation structures between 1000 samples of $\bm \gamma^{(3)}$ from their approximate posterior (the third layer is chosen due to having the smallest number of parameters, but the overall picture is the same for other layers) for LBBNN-GP-MF, LBBNN-GP-MVN, LBBNN-GP-LFMVN models trained on PHONEMNE data, where the same settings as those in Section 4 of the main paper are used. The results are presented in Figures \ref{Fig:amf}-\ref{Fig:alfmvn} and demonstrate a close-to-diagonal correlation structure for LBBNN-GP-MF, and as expected much more structure for LBBNN-GP-MVN and LBBNN-GP-LFMVN. 

\begin{figure}[h!]
\centering

\includegraphics[width=0.8\linewidth]{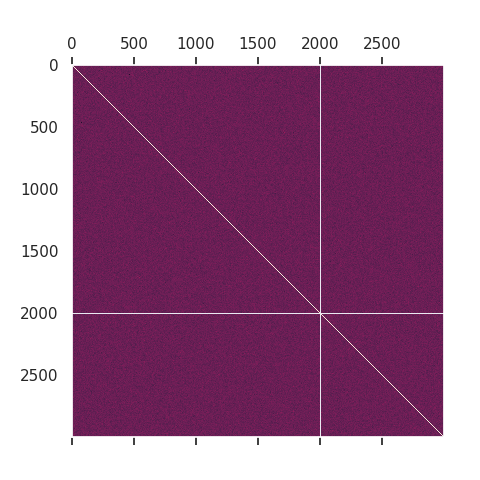}

\caption{\small Correlation structures between posterior inclusions for layer 3 of LBBNN-GP-MF on PHONEMNE data.}\label{Fig:amf}
\end{figure}

\begin{figure}[h!]
\centering

\includegraphics[width=0.8\linewidth]{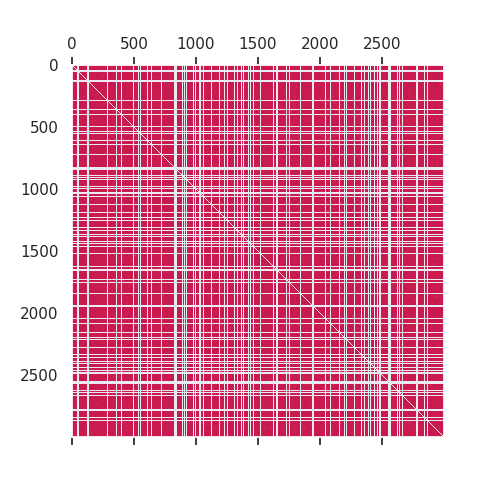}

\caption{\small Correlation structures between posterior inclusions for layer 3 of LBBNN-GP-MVN on PHONEMNE data.}\label{Fig:aMVN}
\end{figure}

\begin{figure}[h!]
\centering

\includegraphics[width=0.8\linewidth]{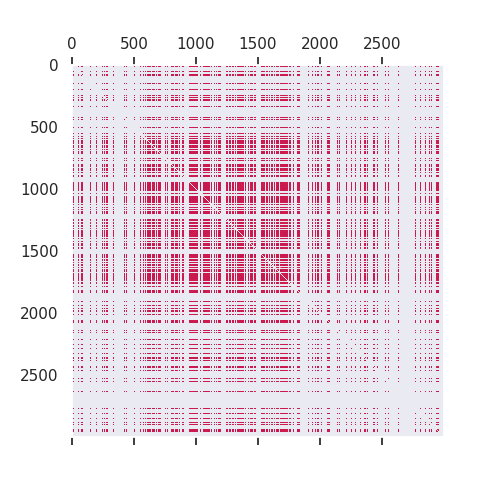}

\caption{\small Correlation structures between posterior inclusions for layer 3 of LBBNN-GP-LFMVN on PHONEMNE data.}\label{Fig:alfmvn}
\end{figure}

\end{document}